\pdfoutput=1
\documentclass[letterpaper]{article} 
\usepackage{aaai2027}  
\usepackage[hyphens]{url}  
\usepackage{graphicx} 
\usepackage{natbib}  
\usepackage{caption} 
\usepackage{booktabs}
\usepackage{multirow}
\usepackage{amsmath}
\usepackage{amssymb}
\usepackage{amsthm}
\newtheorem{theorem}{Theorem}
\newtheorem{proposition}{Proposition}
\newtheorem{lemma}{Lemma}
\newtheorem{corollary}{Corollary}
\theoremstyle{remark}
\newtheorem{remark}{Remark}
\theoremstyle{plain}
\usepackage{placeins}
\usepackage{cuted}    
\usepackage{caption}  
\usepackage{booktabs} 
\usepackage{graphicx} 
\usepackage{algorithm}
\usepackage{algorithmic}

\title{I-SplineFlow: Learning Monotone Spline Stochastic Interpolant Schedulers for Few-Step Generation}
\author{
    Md Sakib Hossain Shovon\textsuperscript{\rm 1}\equalcontrib,
    Md Rifat Ur Rahman\textsuperscript{\rm 2}\equalcontrib,
    Md Abtahi Majeed Chowdhury\textsuperscript{\rm 2}\equalcontrib,\\
    Yunhong Min\textsuperscript{\rm 1},
    Jaesik Choi\textsuperscript{\rm 1},
    Minhyuk Sung\textsuperscript{\rm 1}
}
\affiliations{
    \textsuperscript{\rm 1}KAIST\\
    \textsuperscript{\rm 2}Bangladesh University of Engineering and Technology\\
}

\begin{document}

\maketitle

\begin{abstract}
Few-step generation with pretrained diffusion and flow models can be accelerated
by lightweight training that optimizes the sampling trajectory rather than the
network. A recent approach parameterizes the stochastic interpolant (SI)
scheduler as a smooth curve whose control points enforce the three properties an
SI scheduler must satisfy: fixed boundary conditions, a monotone
signal-to-noise ratio (SNR), and differentiability. Existing parameterizations
use globally supported polynomial bases, where every control point moves the
whole curve and higher expressiveness needs a higher degree, which couples
distant regions of the schedule during optimization. We introduce
\emph{I-SplineFlow}, which parameterizes the scheduler with integrated monotone
splines (I-splines). I-splines decouple the polynomial degree from the number of
mixture weights, so support width and smoothness can be chosen per model at a
fixed weight count, and the compactly supported derivative basis makes the
scheduler Jacobian orders of magnitude better conditioned than a B\'ezier basis.
Boundary conditions and a strictly monotone SNR hold by construction, with no
ordering constraint on the parameters and closed-form velocity derivatives.
Across diffusion (EDM) and flow (ReFlow, Simple ReFlow) models, I-SplineFlow
improves few-step FID over B\'ezier scheduling in most settings, most clearly at
the lowest NFEs, and trains in minutes. Ablations show that both the degree
freedom and the monotonicity constraint are needed. The code will be released upon acceptance.
\end{abstract}

\section{Introduction}

Diffusion and flow-based generative models achieve state-of-the-art sample
quality, but iterative sampling incurs a high number of function evaluations
(NFEs)~\citep{ho2020ddpm,song2021sde,lipman2023flow}. Acceleration ranges from
dedicated ODE solvers~\citep{lu2022dpm,zhao2023unipc,zhang2023ipndm} to
distillation~\citep{song2023consistency,liu2023reflow}, the latter costing
hundreds to thousands of GPU hours.

The \emph{lightweight training} paradigm instead freezes the pretrained model
and learns a small set of parameters under a fixed NFE budget. Most methods
learn a nondecreasing sequence of ODE
timesteps~\citep{tong2025ld3,chen2024gits,xue2024dmn}. \citet{min2025bezierflow}
instead learn the sampling \emph{path} by parameterizing the stochastic
interpolant (SI) scheduler~\citep{albergo2023si}, which transforms the
trajectory geometry while preserving the endpoint marginals and the frozen
network.

An SI scheduler must satisfy three properties: fixed boundary conditions, a
monotone SNR so the induced time map is invertible, and differentiability.
\citet{min2025bezierflow} meet these with a B\'ezier parameterization, but the
Bernstein basis is \emph{globally supported}: its degree is tied to the
control-point count ($\text{degree}=K-1$), so raising expressiveness raises the
degree, which couples distant regions of the schedule and conditions the
optimization poorly.

We introduce \textbf{I-SplineFlow}, which replaces the B\'ezier basis with
\emph{integrated monotone splines} (I-splines). Our contributions are:

\begin{enumerate}
\item \textbf{Freedom of degree and control points.} A spline basis decouples
the polynomial degree $p$ from the number of mixture weights $K$: on a clamped
knot vector $K=p+1+m$ with $m$ interior knots, so at fixed $K$ the degree trades
against interior knots and raising it adds no learnable parameters. The
derivative (M-spline) basis is compactly supported, so each weight controls the
signal/noise density over only a few knot spans, giving local control of the
scheduler \emph{derivative} that a globally supported B\'ezier basis does not
offer. (The integral $I_i$ is not compactly supported and the softmax couples
the weights, so coupling in the scheduler \emph{value} is reduced, not
eliminated.)

\item \textbf{Monotonicity by construction.} I-splines are non-negative,
monotone basis functions with fixed endpoints, so a non-negative (softmax)
mixture yields a strictly monotone SNR and exact boundary conditions
automatically, with no ordering constraint on the raw parameters. All scheduler
derivatives for the transformed velocity remain in closed form, meeting the
differentiability requirement at the $C^{2}$ smoothness the scheduler needs.

\item \textbf{Local support gives better-conditioned optimization.} The
compactly supported derivative basis produces a banded parameter coupling and a
scheduler Jacobian orders of magnitude better conditioned than the B\'ezier
basis (Figure~\ref{fig:conditioning}). This is a property of the basis, not the
optimizer, and holds without value-locality, which neither method has;
empirically the better-conditioned landscape trains faster at a matched budget.
\end{enumerate}

With both properties in place, I-SplineFlow improves few-step FID over B\'ezier
scheduling in most settings across diffusion and flow models, most clearly at
low NFE, and trains in minutes. An otherwise-identical non-monotone B-spline
variant degrades sharply and can diverge on wide-SNR diffusion schedules, which
isolates monotonicity as the operative property.

\section{Related Work}

\paragraph{Learning ODE timesteps.}
Several lightweight methods optimize the sequence of solver timesteps under a
fixed budget: allocating steps to high-curvature
regions~\citep{chen2024gits}, minimizing accumulated integration
error~\citep{xue2024dmn}, or learning timesteps by teacher-forcing
distillation~\citep{tong2025ld3}. All optimize \emph{discrete} per-step
variables on a \emph{fixed} path and generalize poorly to unseen NFEs.

\paragraph{Learning sampling trajectories.}
A second line changes the trajectory itself. Some works select among predefined
SI schedulers at inference~\citep{karras2022edm,pokle2024flows};
\citet{shaul2024bespoke} learn a target path with a discrete per-step
parameterization that models values and derivatives separately, which
destabilizes optimization. Closest to us, \citet{min2025bezierflow} learn a
\emph{continuous} SI scheduler via B\'ezier curves; we adopt this
constraint-by-construction approach and target the global support of the
Bernstein basis, which ties degree to control-point count.

\paragraph{Learning solver coefficients.}
A complementary direction learns the solver's aggregation
weights~\citep{zhao2026dyweight} with unconstrained, time-varying parameters
that relax the sum-to-one constraint. We instead parameterize the scheduler and
argue that an explicit monotonicity constraint helps; we include this comparison
for context and do not evaluate against it.

\paragraph{Scheduler parameterizations.}
Learned noise schedules also appear in variational diffusion
models~\citep{kingma2023vdm}, which model the log-SNR with a monotone network
($>\!10^{3}$ parameters), and multi-marginal
interpolants~\citep{albergo2024multimarginal}, whose trigonometric
parameterization does not guarantee a monotone SNR. Our I-spline parameterization
is compact and satisfies all three requirements by construction.

\section{Method}

\subsection{Background: SI Schedulers and Path Transformation}

Under the stochastic interpolant (SI)
framework~\citep{albergo2023si}, a one-sided interpolant writes the state at
time $t\in[0,1]$ as a linear combination of a noise endpoint $x_0\sim p_0$ and a
data endpoint $x_1\sim p_{\mathrm{data}}$,
\begin{equation}
x(t) = \alpha(t)\,x_1 + \sigma(t)\,x_0,
\label{eq:si}
\end{equation}
where $(\alpha,\sigma)$ is the \emph{scheduler}. Boundary conditions
$\alpha(0)=\sigma(1)=0,\ \alpha(1)=\sigma(0)=1$ fix the endpoints. Given a
pretrained SI model $S_\phi$ trained under a source scheduler
$(\alpha_t,\sigma_t)$, one may sample under a different \emph{target} scheduler
$(\bar\alpha_s,\bar\sigma_s)$ without retraining, by relating the two states
through a scaling reparameterization~\citep{karras2022edm}
$\bar x_s = c_s\,x_{t_s}$. Matching \eqref{eq:si} on both paths forces
\begin{equation}
c_s\,\alpha(t_s)=\bar\alpha(s),\qquad
c_s\,\sigma(t_s)=\bar\sigma(s),
\label{eq:match}
\end{equation}
and dividing the two eliminates $c_s$, giving the SNR-matching time map
\begin{equation}
\rho(t_s)=\bar\rho(s),\quad
t_s=\rho^{-1}\!\big(\bar\rho(s)\big),\quad
\rho(t):=\frac{\alpha(t)}{\sigma(t)},
\label{eq:tmap}
\end{equation}
with scale $c_s=\bar\sigma(s)/\sigma(t_s)$. Here $\rho$ is the SNR, and
$\rho^{-1}$ exists \emph{iff} $\rho$ is strictly monotone. Applying the
change of variables to the source velocity yields the transformed velocity
\begin{equation}
\bar u_s(\bar x_s)=\big(\partial_s\log c_s\big)\,\bar x_s
+ c_s\,\frac{dt_s}{ds}\,u_{t_s}\!\Big(\frac{\bar x_s}{c_s}\Big),
\label{eq:vel}
\end{equation}
whose scalar coefficients $\partial_s\log c_s$ and $dt_s/ds$ depend on the
scheduler and its first derivatives. Learning $(\bar\alpha_s,\bar\sigma_s)$
therefore transforms the sampling-path geometry, and hence the few-step
discretization behavior, but preserves the endpoint marginals and the frozen
network.

\medskip\noindent
Following the teacher-forcing objective of prior
work~\citep{tong2025ld3,min2025bezierflow}, we optimize only the scheduler
parameters $\theta$ to align a few-step student solver
$\bar\xi_\theta(x_0,\{s_i\}_{i=1}^M;S_\phi)$ with a many-step teacher
$\xi(x_0,\{t_i\}_{i=1}^N;S_\phi)$, $M\!\ll\!N$, from shared noise:
\begin{equation}
\min_\theta\ \mathcal{L}(\theta)=
\mathbb{E}_{x_0\sim p_0}\Big[\,d\big(\xi(x_0,\cdots),\ \bar\xi_\theta(x_0,\cdots)\big)\Big],
\label{eq:loss}
\end{equation}
with $d$ a perceptual distance (e.g., LPIPS). The network $S_\phi$ is frozen;
gradients flow through it but never update it.
\subsection{Requirements and the Limitation of Global Bases}

An admissible scheduler must satisfy: (i) \textbf{boundary conditions}
(endpoints fixed), (ii) \textbf{monotonicity} of the SNR $\bar\rho$ (so that
\eqref{eq:tmap} is invertible), and (iii) \textbf{differentiability} (so that
\eqref{eq:vel} is well defined; in practice $C^{2}$).

The B\'ezier parameterization~\citep{min2025bezierflow} satisfies all three via
$K$ ordered control points. But the Bernstein basis is \emph{globally
supported}: for an $n$-degree curve with $n{+}1$ control points,
\begin{equation}
B(\lambda)=\sum_{i=0}^{n} b_{i,n}(\lambda)\,C_i,\quad
b_{i,n}(\lambda)=\binom{n}{i}(1-\lambda)^{n-i}\lambda^{i},
\end{equation}
each $b_{i,n}$ is nonzero on the whole interval $(0,1)$. Two consequences
follow. First, the polynomial degree is \emph{coupled} to the control-point
count ($n=K-1$): more control points force higher degree. Second, every control
point perturbs the entire curve, so optimization cannot adjust one region of the
schedule without affecting all others.

\subsection{I-Spline Stochastic Interpolant Scheduler}

We parameterize each scheduler coefficient with \emph{integrated monotone
splines} (I-splines), which resolve both issues.

\paragraph{Contribution 1: Freedom of degree and control points.}
An I-spline basis $\{I_{i}(s)\}_{i=1}^{K}$ of order $p$ is defined on a knot
vector as the integral of the corresponding M-spline (a non-negative,
normalized B-spline),
\begin{equation}
I_{i}(s)=\int_0^{s} M_{i,p}(u)\,du,\qquad
M_{i,p}(x)=\frac{p+1}{u_{i+p+1}-u_i}\,N_{i,p}(x),
\label{eq:ispline}
\end{equation}
where $N_{i,p}$ is the Cox--de Boor B-spline basis and $p$ denotes the M-spline
\emph{degree} ($p{=}3$ is cubic); each $I_i$ is piecewise degree $p{+}1$. Each
$M_{i,p}$ is \emph{compactly supported}: it is nonzero only on a few knot spans.
Its integral $I_i$ is \emph{not} compactly supported, however: $I_i(s)=0$ before
the support of $M_i$, rises across it, and stays constant (at $1$) afterwards.
A weight $w_i$ therefore controls the scheduler \emph{derivative} $M_i$ locally
while level-shifting the scheduler \emph{value} on the entire tail, so the
locality is in the density, not the value. The number of mixture weights $K$ and
the degree $p$ are independent: for a clamped knot vector $K=p+1+m$ with $m$
interior knots (Proposition~\ref{prop:decouple}), so at fixed $K$ we may vary $p$
by trading it against $m$. Raising $p$ at fixed $K$ adds no learnable parameters
(only $K$, through the weights, does); it widens the basis support and, at the
extreme $m{=}0$, removes locality entirely.

\begin{proposition}[B\'ezier as the no-interior-knot limit]
\label{prop:bezier-limit}
Fix $K$ mixture weights, and let the knot vector have no interior knots, so that
its only knots are $0$ and $1$, each with multiplicity $K$. Then the derivative
basis $\{M_i\}$ equals the degree-$(K{-}1)$ Bernstein basis
$\{b_{i,K-1}\}_{i=0}^{K-1}$ up to positive scaling, so $\dot{\bar\alpha}_\theta$
is a non-negative combination of Bernstein polynomials and $\bar\alpha_\theta$ is
a monotone degree-$K$ Bernstein (B\'ezier-form) curve. In this
$p\to K{-}1$ limit every basis function is supported on all of $(0,1)$, so
locality is lost; the schedule is a globally supported monotone Bernstein form
and need not coincide with a nominal $K$-control-point B\'ezier scheduler, whose
degree would be $K{-}1$ rather than $K$.
\end{proposition}

\begin{proposition}[Degree/control-point decoupling]
\label{prop:decouple}
On a clamped knot vector on $[0,1]$ with M-spline degree $p$ and $m$ simple
interior knots, the number of mixture weights is $K=p+1+m$. Hence for any fixed
$K$ every degree $1\le p\le K-1$ is realizable with $m=K-p-1\ge 0$ interior
knots, so the degree varies independently of the weight count. The endpoint
$m{=}0$ gives $p{=}K{-}1$ and recovers Proposition~\ref{prop:bezier-limit}.
\end{proposition}

\paragraph{Contribution 2: Monotonicity by construction.}
We parameterize each coefficient as a non-negative mixture of I-splines,
\begin{equation}
\bar\alpha_\theta(s)=\sum_{i=1}^{K} w_i^{(\alpha)}\,I_i(s),\qquad
\bar\sigma_\theta(s)=1-\sum_{i=1}^{K} w_i^{(\sigma)}\,I_i(s),
\label{eq:coeffs}
\end{equation}
with weights $w^{(\cdot)}=\mathrm{softmax}(\theta^{(\cdot)})$, so $w_i>0$ and
$\sum_i w_i=1$. Because each $I_i$ is monotone nondecreasing with $I_i(0)=0$ and
$I_i(1)=1$, this construction guarantees:
\begin{itemize}
\item \emph{Boundary conditions:} $\bar\alpha(0)=0,\ \bar\alpha(1)=1$
(and mirrored for $\bar\sigma$), since $\sum_i w_i=1$.
\item \emph{Monotone SNR:} $\dot{\bar\alpha}=\sum_i w_i M_i\ge 0$ and
$\dot{\bar\sigma}\le 0$ with $w_i>0$, hence
\begin{equation}
\dot{\bar\rho}
=\frac{\dot{\bar\alpha}\,\bar\sigma-\bar\alpha\,\dot{\bar\sigma}}
{\bar\sigma^{2}}>0,
\end{equation}
so $\bar\rho$ is strictly increasing and $\bar\rho^{-1}$ exists.
\end{itemize}

\begin{theorem}[Admissibility of the I-Spline scheduler]
\label{thm:admissible}
Let $\{I_i\}_{i=1}^{K}$ be the order-$p$ I-spline basis of \eqref{eq:ispline}
on a clamped knot vector on $[0,1]$ with simple interior knots, and let
$\bar\alpha_\theta,\bar\sigma_\theta$ be given by \eqref{eq:coeffs} with
$w^{(\alpha)}{=}\mathrm{softmax}(\theta^{(\alpha)})$ and
$w^{(\sigma)}{=}\mathrm{softmax}(\theta^{(\sigma)})$. Then for every
$\theta^{(\alpha)},\theta^{(\sigma)}\in\mathbb{R}^{K}$:
\begin{enumerate}
\item[(i)] $\bar\alpha(0){=}0,\ \bar\alpha(1){=}1,\ \bar\sigma(0){=}1,\ \bar\sigma(1){=}0$;
\item[(ii)] the SNR $\bar\rho=\bar\alpha/\bar\sigma$ is strictly increasing on
$[0,1)$, so $\bar\rho^{-1}$ exists;
\item[(iii)] for any degree $p\ge 2$, $\bar\alpha,\bar\sigma\in C^{p}$
(hence in particular $C^{2}$).
\end{enumerate}
Every $\theta$ therefore yields an admissible scheduler in the sense of
requirements~(i)--(iii) of the previous subsection, with no ordering constraint
on the parameters. The proof is given in the supplementary material.
\end{theorem}

Unlike the ordered-control-point construction of B\'ezier, monotonicity here
requires \emph{no ordering constraint} on $\theta$: any real $\theta$ yields a
valid, monotone, boundary-correct scheduler. The constraint is absorbed into the
basis rather than enforced during optimization.

\paragraph{Differentiability and closed-form velocity.}
The derivative needed by \eqref{eq:vel} is available in closed form, since the
derivative of an I-spline is exactly its M-spline:
$\dot{\bar\alpha}_\theta(s)=\sum_i w_i M_i(s)$. Thus the scalar coefficients
$\partial_s\log c_s$ and $dt_s/ds$ (via the inverse-function theorem on
\eqref{eq:tmap}) are computed analytically, without the separately-learned
derivatives that destabilize discrete parameterizations~\citep{shaul2024bespoke}.

\paragraph{Initialization.}
We initialize to the linear scheduler $\bar\alpha(s)=s,\ \bar\sigma(s)=1-s$ by
choosing weights that make $\dot{\bar\alpha}$ constant, matching the standard
starting point of prior work.
\subsection{The Role of Monotonicity: a Controlled Ablation}
To isolate which of the two properties drives performance, we introduce a
\emph{B-spline} counterpart. It shares the entire basis machinery, knots,
degree, and initialization of the I-spline, so it keeps the degree/control-point
freedom of Contribution~1, but it replaces the non-negative I-spline mixture with
free control points,
\begin{equation}
\bar\alpha_\theta(s)=\sum_{i} c_i^{(\alpha)} N_{i,p}(s),\quad
c_0^{(\alpha)}{=}0,\ c_{K-1}^{(\alpha)}{=}1,
\end{equation}
so the curve is \emph{not} constrained to be monotone. Comparing B-spline to
I-spline under identical settings probes the effect of monotonicity
(Contribution~2); comparing B-spline to B\'ezier probes degree/control-point
decoupling (Contribution~1). We report this ablation in
the experiments section (Table~\ref{tab:ablation-mono}). The following
proposition states why a non-monotone SNR is problematic; its role is to
motivate the constraint, not to fully attribute the empirical gap.

\begin{proposition}[Monotonicity and valid reparameterization]
\label{prop:reparam}
Let $\rho$ be the strictly increasing source SNR and $\bar\rho$ the target SNR,
with time map $t_s=\rho^{-1}(\bar\rho(s))$. Then $dt_s/ds=\bar\rho'(s)/\rho'(t_s)$
has the sign of $\bar\rho'(s)$. If $\bar\rho$ is not strictly increasing, then at
points where $\bar\rho'(s)=0$ the source time stalls and where $\bar\rho'(s)<0$
it reverses and $s\mapsto t_s$ fails to be injective, so the sampling path is no
longer a valid orientation-preserving reparameterization of the source path.
Strict monotonicity of $\bar\rho$ (Theorem~\ref{thm:admissible}) makes $s\mapsto
t_s$ a strictly increasing bijection.
\end{proposition}

\begin{strip}
\centering
\captionof{table}{FID of few-step generation with flow-based models and diffusion models (EDM). Lower is better; best per column in \textbf{bold}. Top row of each block is the base ODE solver. I-SplineFlow uses polynomial degree $3$ for flow-based models and $16$ for diffusion models.}
\label{tab:combined-fid}
\setlength{\tabcolsep}{4.5pt}
\scalebox{0.8}{
\begin{tabular}{l | cccc | l | cccc}
\toprule
Method & NFE=4 & NFE=6 & NFE=8 & NFE=10 & Method & NFE=4 & NFE=6 & NFE=8 & NFE=10 \\
\midrule
\multicolumn{10}{c}{\textit{CIFAR-10 $32\times32$ with ReFlow (Teacher FID: 2.70)}} \\
\midrule
RK1             & 52.78 & 26.30 & 17.40 & 13.30 & RK2             & 25.36 & 12.12 & 9.17  & 7.89 \\
+ DMN           & 180.03 & 104.23 & 30.94 & 21.58 & + DMN           & 82.41 & 51.99 & 21.43 & 18.62 \\
+ Bespoke       & 45.31 & 18.08 & 11.88 & 9.25  & + Bespoke       & 39.45 & 64.87 & 16.67 & 13.34 \\
+ GITS          & 47.42 & 26.11 & 19.89 & 15.34 & + GITS          & 22.84 & 11.84 & 8.77  & 6.58 \\
+ B\'ezierFlow  & 20.65 & 9.69  & 7.32  & 5.53  & + B\'ezierFlow  & 13.20 & 6.03  & 4.33  & 3.75 \\
+ I-SplineFlow  & \textbf{20.29} & \textbf{9.29}  & \textbf{7.19}  & \textbf{5.31} & + I-SplineFlow  & \textbf{13.09} & \textbf{5.35}  & \textbf{3.65}  & \textbf{3.20} \\
\midrule
\multicolumn{10}{c}{\textit{CIFAR-10 $32\times32$ with ReFlow++ (Teacher FID: 2.33)}} \\
\midrule
RK1             & 2.65  & 2.64  & 2.63  & 2.60  & RK2             & 2.50  & 2.49  & 2.50  & 2.50 \\
+ DMN           & 2.67  & 2.63  & 2.62  & 2.60  & + DMN           & 2.50  & 2.49  & 2.51  & 2.51 \\
+ Bespoke       & 23.64 & 7.44  & 3.35  & 2.75  & + Bespoke       & 47.91 & 20.75 & 7.45  & 6.48 \\
+ GITS          & 2.63  & 2.62  & 2.60  & 2.58  & + GITS          & 2.50  & 2.50  & 2.52  & 2.52 \\
+ B\'ezierFlow  & 2.72  & 2.61  & 2.59  & 2.58  & + B\'ezierFlow  & 2.46  & \textbf{2.47}  & \textbf{2.49}  & 2.40 \\
+ I-SplineFlow  & \textbf{2.61}  & \textbf{2.58}  & \textbf{2.57}  & \textbf{2.54}  & + I-SplineFlow  & \textbf{2.45}  & \textbf{2.47}  & \textbf{2.49}  & \textbf{2.38} \\
\midrule
\multicolumn{10}{c}{\textit{ImageNet $256\times256$ with FlowDCN (Teacher FID: 15.89)}} \\
\midrule
RK1             & 12.03 & 12.04 & 13.55 & 14.43 & RK2             & 7.91  & 10.54 & 12.97 & 14.08 \\
+ DMN           & 142.79 & 28.56 & 10.61 & 11.69 & + DMN           & 7.96  & 10.23 & 9.42  & 7.86 \\
+ Bespoke       & \textbf{11.85} & 11.81 & 13.39 & 14.31 & + Bespoke       & \textbf{7.66}  & 10.05 & 13.02 & 14.23 \\
+ GITS          & 13.20 & 10.91 & 11.91 & 12.93 & + GITS          & 8.18  & 9.80  & 12.30 & 13.27 \\
+ B\'ezierFlow  & 15.61 & 6.87  & 7.79  & 8.13  & + B\'ezierFlow  & 9.52  & 5.97  & 6.24  & 7.57 \\
+ I-SplineFlow  & 15.58 & \textbf{6.82}  & \textbf{7.77}  & \textbf{8.09}  & + I-SplineFlow  & 9.50  & \textbf{5.94}  & \textbf{6.19}  & \textbf{7.53} \\
\midrule
\multicolumn{10}{c}{\textit{CIFAR-10 $32\times32$ with EDM (Teacher FID: 2.08)}} \\
\midrule
UniPC           & 50.55 & 19.59 & 10.02 & 6.48  & iPNDM           & 29.11 & 10.51 & 5.23  & 3.72 \\
+ DMN           & 26.67 & 8.36  & 4.63  & 3.12  & + DMN           & 27.91 & 10.05 & 4.79  & 3.49 \\
+ GITS          & 25.10 & 11.29 & 7.05  & 5.31  & + GITS          & 15.82 & 7.48  & 4.04  & 3.25 \\
+ B\'ezierFlow  & 9.81  & 3.39  & 2.79  & 2.40  & + B\'ezierFlow  & 6.53  & 4.04  & \textbf{2.76} & \textbf{2.39} \\
+ I-SplineFlow  & \textbf{8.87}  & \textbf{3.27}  & \textbf{2.44}  & \textbf{2.38}  & + I-SplineFlow  & \textbf{6.47}  & \textbf{3.62}  & 2.93  & 2.45 \\
\midrule
\multicolumn{10}{c}{\textit{FFHQ $64\times64$ with EDM (Teacher FID: 2.86)}} \\
\midrule
UniPC           & 52.28 & 14.88 & 7.82  & 9.14  & iPNDM           & 29.91 & 11.58 & 6.42  & 5.26 \\
+ DMN           & 30.56 & 9.35  & 5.12  & 4.29  & + DMN           & 32.05 & 12.39 & 7.05  & 6.65 \\
+ GITS          & 27.69 & 12.01 & 8.98  & 4.61  & + GITS          & 19.71 & 9.67  & 5.32  & 4.85 \\
+ B\'ezierFlow  & 21.73 & 7.33  & 3.89  & 3.35  & + B\'ezierFlow  & 16.57 & \textbf{8.29}  & \textbf{5.29}  & 4.22 \\
+ I-SplineFlow  & \textbf{17.46} & \textbf{6.67}  & \textbf{3.70}  & \textbf{3.06}  & + I-SplineFlow  & \textbf{16.47} & 9.06  & 5.53  & \textbf{4.20} \\
\midrule
\multicolumn{10}{c}{\textit{AFHQv2 $64\times64$ with EDM (Teacher FID: 2.04)}} \\
\midrule
UniPC           & 24.52 & 10.83 & 8.11  & 6.21  & iPNDM           & 14.63 & 7.61  & 4.21  & 3.49 \\
+ DMN           & 31.31 & 15.02 & 4.33  & 3.52  & + DMN           & 32.71 & 17.36 & 6.47  & 5.71 \\
+ GITS          & 14.05 & 8.18  & 4.24  & 3.78  & + GITS          & 13.72 & 7.30  & 4.29  & 4.05 \\
+ B\'ezierFlow  & 13.16 & 5.11  & 3.07  & 2.53  & + B\'ezierFlow  & 13.89 & 6.14  & 3.07  & 2.78 \\
+ I-SplineFlow  & \textbf{13.04} & \textbf{5.02}  & \textbf{2.87}  & \textbf{2.48}  & + I-SplineFlow  & \textbf{12.35} & \textbf{5.96}  & \textbf{2.89}  & \textbf{2.69} \\
\bottomrule
\end{tabular}
}
\end{strip}

\section{Experiments}
\label{sec:experiments}

We evaluate on diffusion and flow models. For diffusion, EDM~\citep{karras2022edm}
on CIFAR-10 ($32{\times}32$), FFHQ and AFHQv2 ($64{\times}64$), with
UniPC~\citep{zhao2023unipc} and iPNDM~\citep{zhang2023ipndm}. For flow,
ReFlow~\citep{liu2023reflow} and Simple ReFlow~\citep{kim2025simplereflow} on
CIFAR-10 and FlowDCN~\citep{wang2024flowdcn} on ImageNet ($256{\times}256$), with
RK1 (Euler) and RK2 (Midpoint). Following B\'ezierFlow~\citep{min2025bezierflow},
teachers use a high-order adaptive solver from shared noise and we train an LPIPS
objective on matched budgets (200 pairs for CIFAR-10, 50 otherwise). All runs use
$K{=}32$ and RMSprop; FID uses 50K samples against the EDM reference statistics.

\paragraph{Diffusion models.}
I-SplineFlow achieves the best FID in most settings across the three datasets and
beats B\'ezierFlow in every column on CIFAR-10 (UniPC), FFHQ (UniPC), and AFHQv2
(both solvers), as shown in Table~\ref{tab:combined-fid}. With iPNDM the gain
concentrates at the lowest NFEs (6.47 vs.\ 6.53 on CIFAR-10, 12.35 vs.\ 13.89 on
AFHQv2 at NFE${=}4$), where discretization error dominates.
\paragraph{Flow models.}
On ReFlow, I-SplineFlow is best in every column, with the largest low-NFE gains under RK2 (5.35 vs. 6.03 at NFE=6). On the near-straight ReFlow++ trajectories all methods sit within a narrow band, yet I-SplineFlow still attains the best FID at every NFE. On ImageNet with FlowDCN it leads at NFE 6–10; at NFE 4 Bespoke is lowest, where the curved high-resolution trajectory favors a per-step solver. Degree 3 suffices throughout, since little curvature remains to fit.

\begin{strip}
\centering
\includegraphics[width=0.8\textwidth]{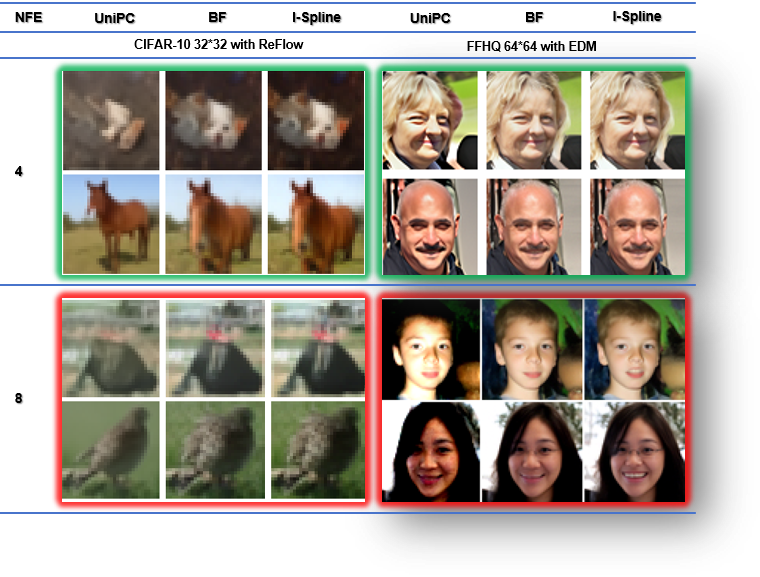}
\captionof{figure}{Few-step samples from shared noise on CIFAR-10 (ReFlow) and FFHQ
(EDM). In each panel the columns are the base ODE solver, B\'ezierFlow (BF),
and I-SplineFlow; rows are NFE $4$ (top) and NFE $8$ (bottom).}
\label{fig:sample}
\end{strip}

\subsection{Ablation: Degree--Control-Point Decoupling}

With $K{=}32$ fixed we vary $p\in\{3,6,10,16\}$; degree changes only the knot
vector, so every row of Table~\ref{tab:ablation-degree} has the same parameter
count. Since $K=p+1+m$ (Proposition~\ref{prop:decouple}), raising $p$ removes
interior knots ($m=28,25,21,15$) and widens support, so the sweep trades degree
against locality. FID improves with degree up to $p{=}16$, the strongest setting
on nearly all datasets. The trend has a ceiling: at $p{=}31$ the knots vanish and
the spline becomes a globally supported Bernstein curve
(Proposition~\ref{prop:bezier-limit}), so the sweep interpolates between local
($p{=}3$) and global bases. The interior optimum reflects a
conditioning-versus-expressivity trade-off: conditioning improves as degree drops
(Figure~\ref{fig:conditioning}) but too local a basis underfits. Across all eight
solver$\times$NFE configurations the FID penalty above each configuration's best
degree is smallest near $p{=}16$ ($2.3\%$) and the LPIPS penalty near $p{=}20$,
both rising to $14.0\%$ and $23.3\%$ at $p{=}31$
(Figures~\ref{fig:degree_penalty},~\ref{fig:degree_fid}). Training cost tracks the
NFE budget, not the degree: under $5\%$ variation across $p$ ($\le 8$ minutes)
versus $2.2\times$ from NFE $4$ to $10$ (Figure~\ref{fig:degree_time}). We adopt
$p{=}16$.

\begin{figure}[h]
    \includegraphics[width=0.5\textwidth]{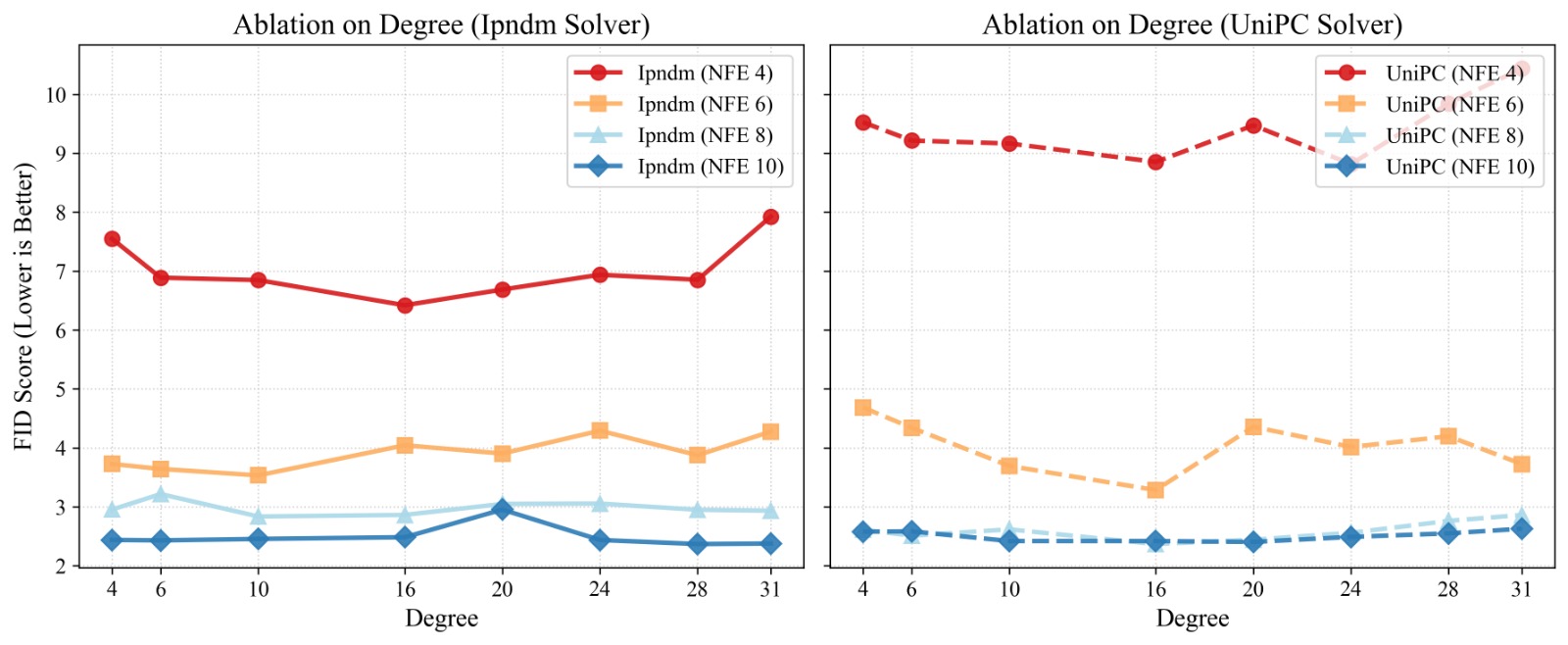}
    \caption{FID against I-spline degree for iPNDM (left) and UniPC (right) at
    NFE $4$--$10$, extending Table~\ref{tab:ablation-degree} to $p{=}31$. Each
    curve dips at an interior degree and rises toward the global limit.}
    \label{fig:degree_fid}
\end{figure}

\vspace{1em}

\small
\renewcommand{\arraystretch}{1.00}
\captionof{table}{Effect of I-Spline degree $p$ at fixed $K{=}32$ (control points). B\'ezierFlow shown as a degreeless reference. FID, lower is better; best per row in \textbf{bold}. The Base column indicates the base ODE solver.}
\label{tab:ablation-degree}
\setlength{\tabcolsep}{4.5pt}
\begin{tabular}{l c c c cccc}
\toprule
& & & & \multicolumn{4}{c}{I-Spline Degree $p$} \\
\cmidrule(lr){5-8}
Solver & NFE & Base & B\'ezierFlow & 3 & 6 & 10 & 16 \\
\midrule
\multicolumn{8}{c}{\textit{CIFAR-10 $32\times32$ with EDM (Teacher FID: 2.08)}} \\
\midrule
UniPC 
& 4  & 50.55 & 9.81  & 9.01  & 9.25  & 9.21  & \textbf{8.87} \\
& 6  & 19.59 & 3.39  & 3.69  & 4.43  & 3.77  & \textbf{3.27} \\
& 8  & 10.02 & 2.79  & 2.90  & 2.59  & 2.69  & \textbf{2.44} \\
& 10 & 6.48  & 2.40  & 2.63  & 2.60  & 2.43  & \textbf{2.38} \\
\midrule
iPNDM 
& 4  & 29.11 & 6.53  & 8.15  & 6.98  & 6.85  & \textbf{6.47} \\
& 6  & 10.51 & 4.04  & 4.44  & 3.74  & 3.69  & \textbf{3.62} \\
& 8  & 5.23  & \textbf{2.76} & 2.94  & 3.25  & 2.88  & 2.93 \\
& 10 & 3.72  & \textbf{2.39} & 2.47  & 2.48  & 2.49  & 2.45 \\
\midrule
\multicolumn{8}{c}{\textit{FFHQ $64\times64$ with EDM (Teacher FID: 2.86)}} \\
\midrule
UniPC 
& 4  & 52.28 & 21.73 & 18.42 & 18.52 & 18.65 & \textbf{17.46} \\
& 6  & 14.88 & 7.33  & 6.90  & 6.98  & 7.05  & \textbf{6.67} \\
& 8  & 7.82  & 3.89  & 3.82  & 3.81  & 3.80  & \textbf{3.70} \\
& 10 & 9.14  & 3.35  & 3.07  & 3.11  & 3.15  & \textbf{3.06} \\
\midrule
iPNDM 
& 4  & 29.91 & 16.57 & 19.47 & 17.65 & \textbf{15.97} & 16.47 \\
& 6  & 11.58 & \textbf{8.29}  & 9.41  & 9.32  & 9.25  & 9.06 \\
& 8  & 6.42  & \textbf{5.29}  & 5.67  & 5.52  & 5.41  & 5.53 \\
& 10 & 5.26  & 4.22  & 4.26  & 4.25  & 4.25  & \textbf{4.20} \\
\midrule
\multicolumn{8}{c}{\textit{AFHQv2 $64\times64$ with EDM (Teacher FID: 2.04)}} \\
\midrule
UniPC 
& 4  & 24.52 & 13.16 & 13.38 & 13.30 & 13.22 & \textbf{13.04} \\
& 6  & 10.83 & 5.11  & 5.28  & 5.55  & 5.87  & \textbf{5.02} \\
& 8  & 8.11  & 3.07  & 3.03  & 2.95  & 2.98  & \textbf{2.87} \\
& 10 & 6.21  & 2.53  & 2.67  & 2.68  & 2.69  & \textbf{2.48} \\
\midrule
iPNDM 
& 4  & 14.63 & 13.89 & 14.04 & 13.12 & \textbf{12.33} & 12.35 \\
& 6  & 7.61  & 6.14  & \textbf{5.51}  & 5.84  & 6.23  & 5.96 \\
& 8  & 4.21  & 3.07  & 3.21  & 3.14  & 3.07  & \textbf{2.89} \\
& 10 & 3.49  & 2.78  & 2.73  & 2.72  & 2.72  & \textbf{2.69} \\
\bottomrule
\end{tabular}

\begin{figure}[h]
    \centering
    \includegraphics[width=0.8\columnwidth]{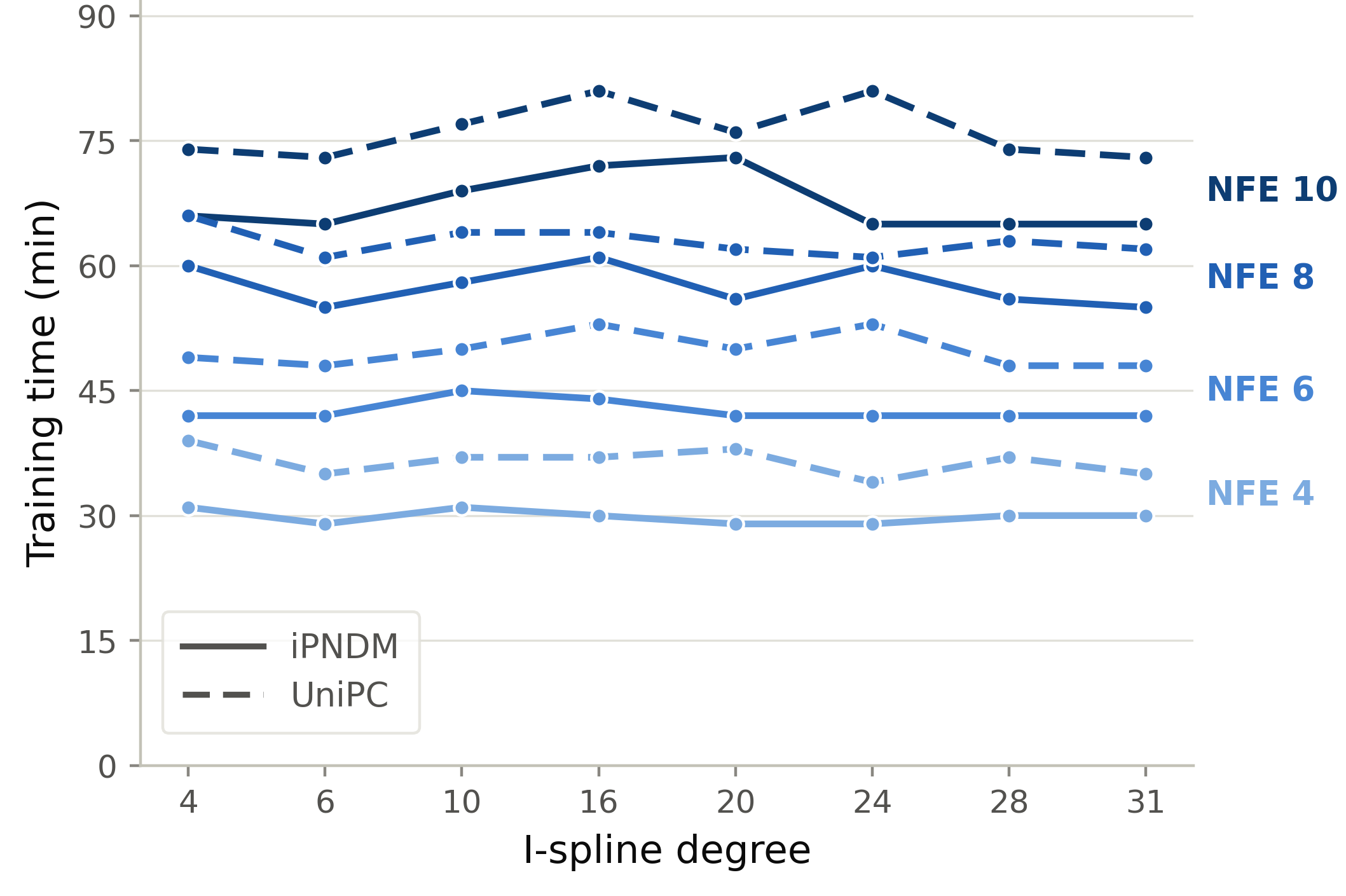}
    \caption{Training cost is set by the sampling budget, not the degree. Across
    $p\in\{4,\dots,31\}$ the per-configuration wall-clock time varies by under
    $5\%$ ($\le 8$ minutes), while it grows about $2.2\times$ from NFE $4$ to
    $10$.}
    \label{fig:degree_time}
\end{figure}

\subsection{Locality and Conditioning}

\begin{figure}[h]
    \centering
    \includegraphics[width=0.8\columnwidth]{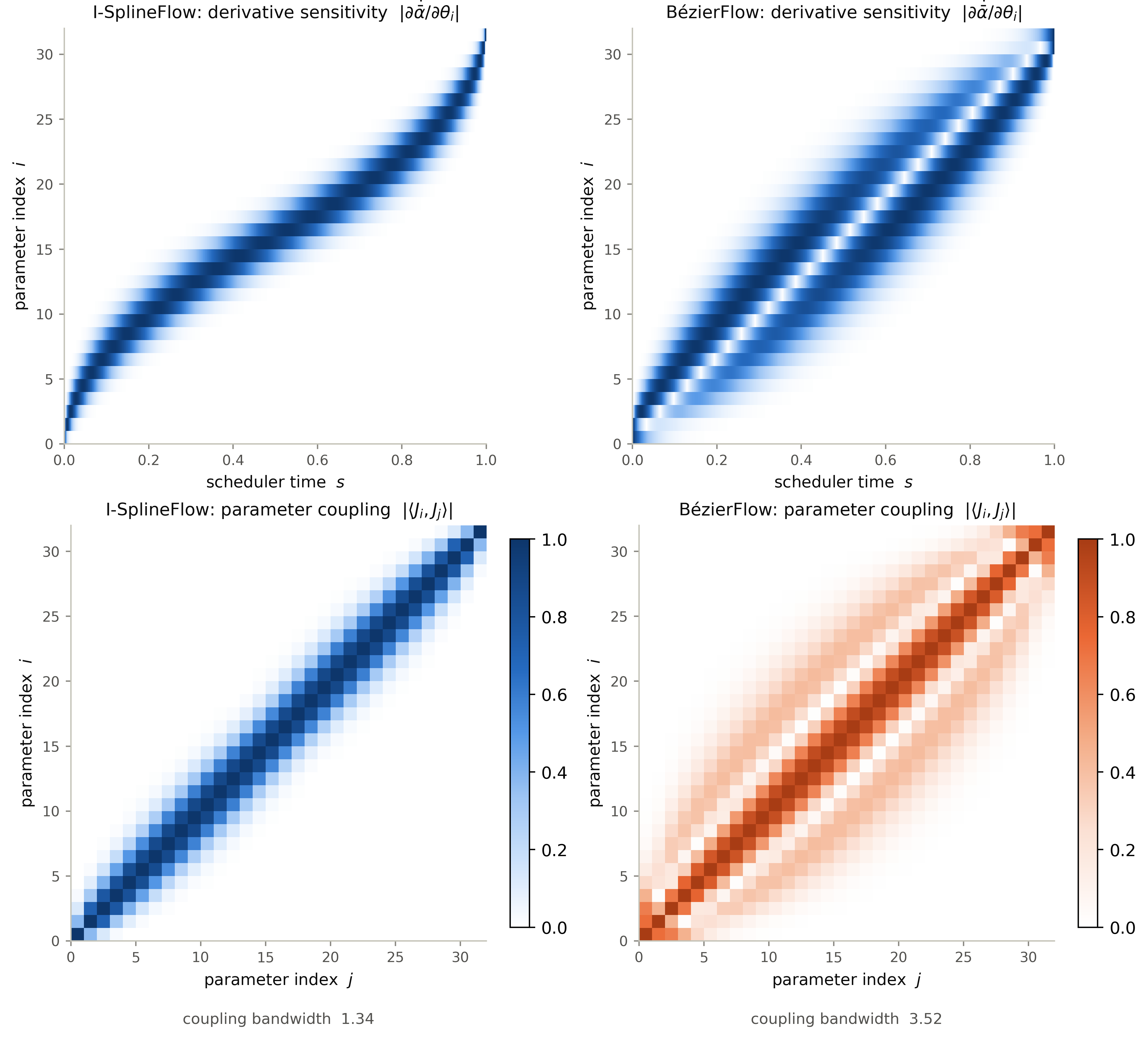}
    \caption{Local support yields sparse, banded parameter coupling. Top:
    derivative sensitivity $|\partial\dot{\bar\alpha}(s)/\partial\theta_i|$ over
    parameter index $i$ and time $s$, per row normalized; the I-spline basis is
    nonzero on only a few knot spans, while the Bernstein derivative has a
    broader footprint. Bottom: parameter coupling
    $|\langle J_i,J_j\rangle|/(\|J_i\|\,\|J_j\|)$, banded for I-SplineFlow
    (bandwidth $1.34$) and broad for B\'ezierFlow ($3.52$). Degree-16 operating
    point.}
    \label{fig:coupling}
\end{figure}

\begin{figure}[h]
    \centering
    \includegraphics[width=0.5\textwidth]{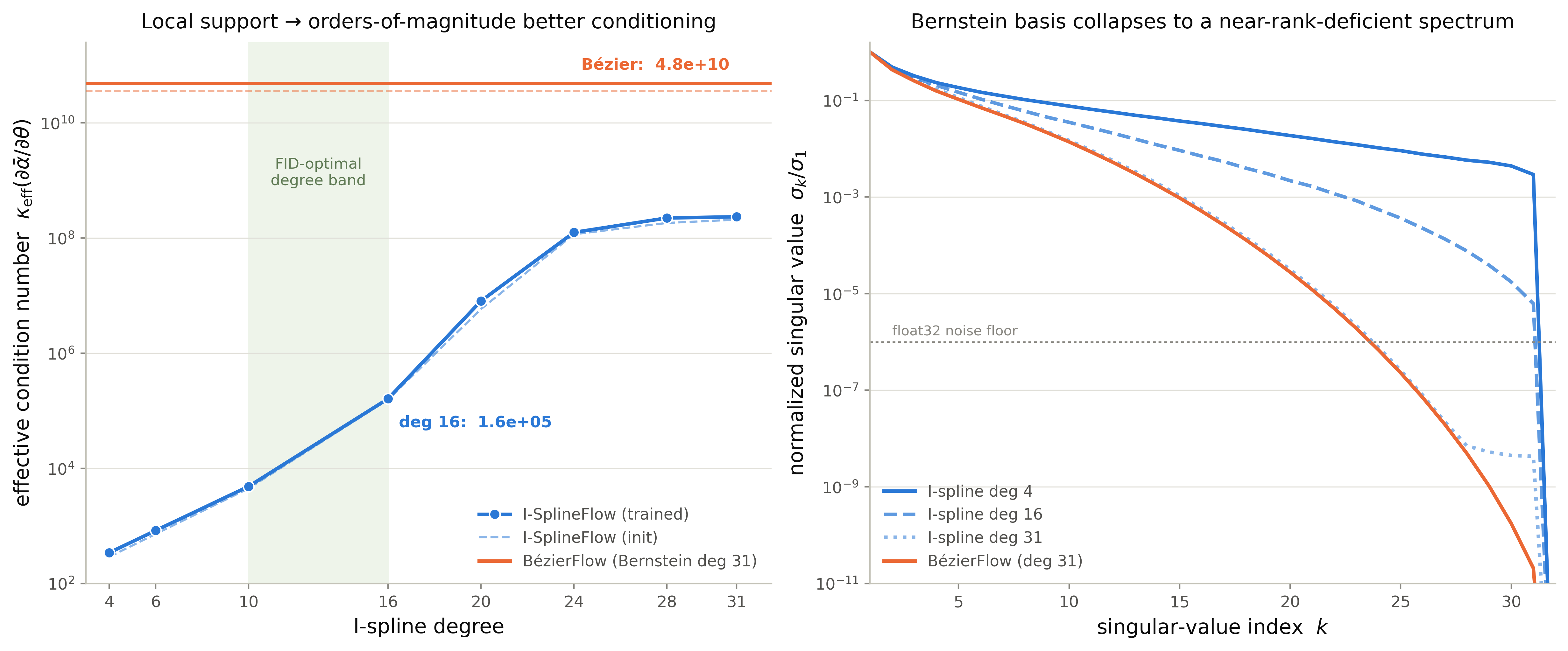}
    \caption{Local support yields a better-conditioned scheduler. The left panel
    plots the effective condition number
    $\kappa_{\mathrm{eff}}=\sigma_1/\sigma_{K-1}$ of
    $\partial\bar\alpha/\partial\theta$ against I-spline degree, at
    initialization and at the trained checkpoint, with B\'ezierFlow (degree
    $31$) as the horizontal reference; $\kappa_{\mathrm{eff}}$ grows with degree
    toward the global Bernstein limit and is about five orders of magnitude
    smaller than B\'ezierFlow at the operating degree, and it is nearly identical
    at initialization and after training. The shaded band marks the FID-optimal
    degree range of Table~\ref{tab:ablation-degree}. The right panel plots the
    normalized singular-value spectra of the same Jacobian: a low-degree spline
    keeps a full spectrum, while high degree and the Bernstein basis collapse
    toward rank deficiency.}
    \label{fig:conditioning}
\end{figure}

We verify locality on the trained EDM/CIFAR-10 checkpoints through the
optimizer-level Jacobian $\partial\bar\alpha/\partial\theta$, computed
analytically and matching autograd to $2{\times}10^{-4}$. The compactly supported
M-spline derivative moves the density $\dot{\bar\alpha}=\sum_i w_i M_i$ over a few
knot spans, while the Bernstein derivative spans all of $(0,1)$, so the parameter
coupling is banded for I-SplineFlow and broad for B\'ezierFlow ($1.34$ vs.\
$3.52$, Figure~\ref{fig:coupling}). The value map is not local for either
(participation ratio $\approx 0.73$): locality lives in the derivative.

This banded coupling makes the scheduler Jacobian far better conditioned. Its
effective condition number is about five orders of magnitude smaller than
B\'ezierFlow at our operating point ($1.6{\times}10^{5}$ vs.\
$4.8{\times}10^{10}$, Figure~\ref{fig:conditioning}); it is nearly identical at
initialization and after training, so it is a property of the basis rather than
of where the optimizer sits, and it grows with degree as the spline approaches
the global Bernstein limit, matching the known increase of the Bernstein form's
condition number with degree~\citep{farouki1987bernstein} that a compact,
overlap-limited spline basis avoids~\citep{deboor2001splines}.

The better-conditioned landscape trains faster at a matched budget
(Figure~\ref{fig:convergence}): I-SplineFlow reaches B\'ezierFlow's fully trained
validation LPIPS about three epochs earlier and settles $6$--$15\%$ lower in the
featured few-step settings. Per-degree tables and the full derivations are in the
supplementary material.

\begin{figure}[h]
    \centering
    \includegraphics[width=0.5\textwidth]{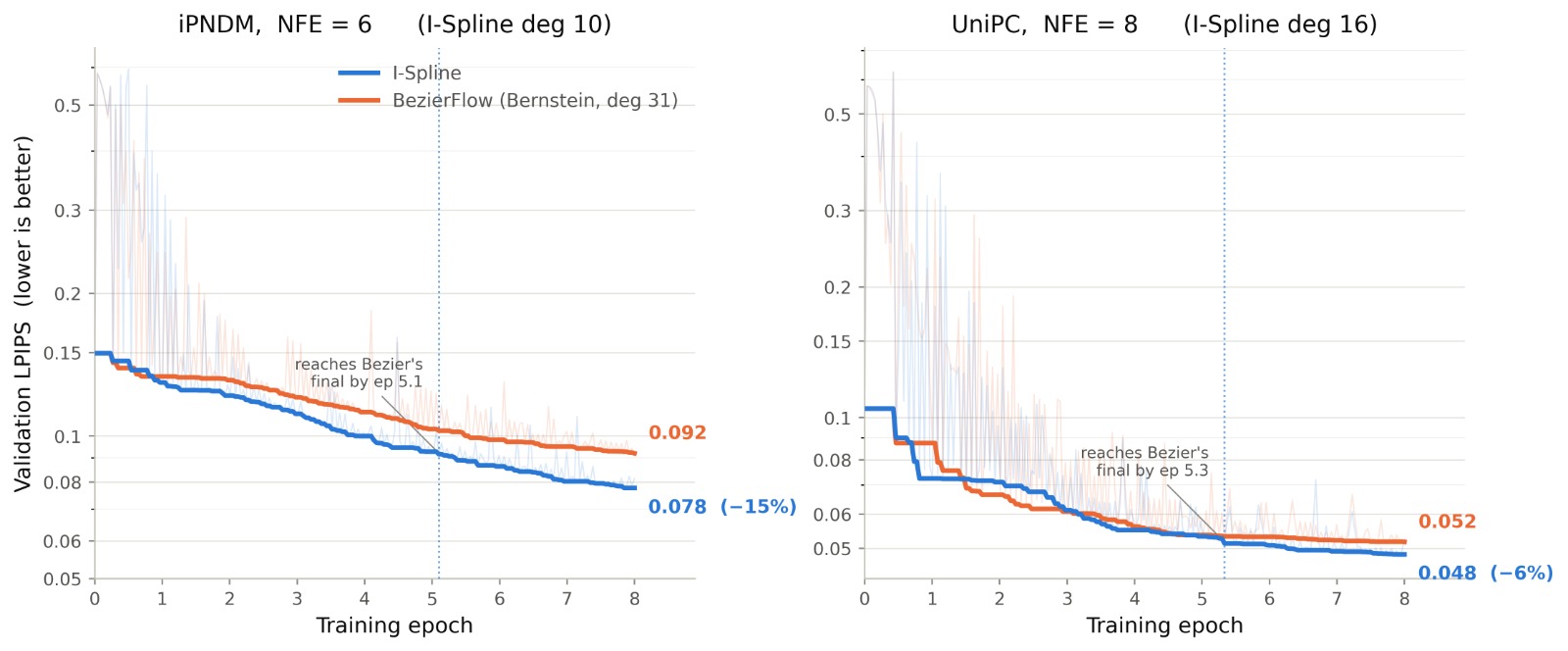}
    \caption{Better conditioning gives faster, lower convergence. Validation
    LPIPS (running best) against training epoch at a matched budget, same
    $K{=}32$, same initialization, and same teacher data. I-SplineFlow reaches
    B\'ezierFlow's fully trained value about three epochs earlier and settles
    $6$--$15\%$ lower in both featured few-step settings.}
    \label{fig:convergence}
\end{figure}

\subsection{Ablation: The Role of Monotonicity}

We compare against a B-spline counterpart with the same basis, knots, degree,
initialization, and budget but free (non-monotone) control points. It keeps
Contribution~1 and drops Contribution~2, though not in isolation, since the free
points and softmax also change the feasible set. I-Spline beats it at matched
degree in nearly every column (Table~\ref{tab:ablation-mono}), most sharply at
iPNDM NFE${=}4$ (6.47--6.85 vs.\ 17.72--18.14), where the unconstrained SNR loses
monotonicity and the reparameterization stalls or reverses
(Proposition~\ref{prop:reparam}). A monotone-B-spline control would isolate
monotonicity from the parameterization, which we leave to future work.
\small
\renewcommand{\arraystretch}{1.00}
\captionof{table}{Monotonicity ablation at matched degree, knots, and parameter budget.
FID, lower is better; best per column in \textbf{bold}. Top row of each block
is the base ODE solver.}
\label{tab:ablation-mono}
\setlength{\tabcolsep}{4.5pt}
\begin{tabular}{l cccc}
\toprule
& \multicolumn{4}{c}{NFE} \\
\cmidrule(lr){2-5}
Method & 4 & 6 & 8 & 10 \\
\midrule
\multicolumn{5}{c}{\textit{CIFAR-10 $32\times32$ with EDM (Teacher FID: 2.08)}} \\
\midrule
UniPC             & 50.55 & 19.59 & 10.02 & 6.48 \\
+ B\'ezierFlow    & 9.81  & 3.39  & 2.79  & 2.40 \\
+ B-Spline $p{=}10$ & 11.31 & 3.73 & 2.70 & 2.68 \\
+ B-Spline $p{=}16$ & 10.90 & 3.65  & \textbf{2.25} & 2.39 \\
+ I-Spline $p{=}10$ & 9.21  & 3.77  & 2.69  & 2.43 \\
+ I-Spline $p{=}16$ & \textbf{8.87} & \textbf{3.27}  & 2.44  & \textbf{2.38} \\
\midrule
iPNDM             & 29.11 & 10.51 & 5.23  & 3.72 \\
+ B\'ezierFlow    & 6.53  & 4.04  & \textbf{2.76} & \textbf{2.39} \\
+ B-Spline $p{=}10$ & 18.14 & 4.27  & 2.97  & 2.69 \\
+ B-Spline $p{=}16$ & 17.72 & 4.08  & 2.91  & 2.41 \\
+ I-Spline $p{=}10$ & 6.85  & \textbf{3.62} & 2.88  & 2.49 \\
+ I-Spline $p{=}16$ & \textbf{6.47} & \textbf{3.62} & 2.93  & 2.45 \\
\midrule
\multicolumn{5}{c}{\textit{CIFAR-10 $32\times32$ with ReFlow (Teacher FID: 2.70)}} \\
\midrule
RK1               & 52.78 & 26.30 & 17.40 & 13.30 \\
+ B\'ezierFlow    & 20.65 & 9.69  & 7.32  & 5.53 \\
+ B-Spline $p{=}3$  & \textbf{19.15} & 10.43 & 7.30  & 6.21 \\
+ I-Spline $p{=}3$  & 20.29 & \textbf{9.29} & \textbf{7.19} & \textbf{5.31} \\
\midrule
RK2               & 25.36 & 12.12 & 9.17  & 7.89 \\
+ B\'ezierFlow    & 13.20 & 6.03  & 4.33  & 3.75 \\
+ B-Spline $p{=}3$  & 15.00 & 8.74  & 5.38  & 3.61 \\
+ I-Spline $p{=}3$  & \textbf{13.09} & \textbf{5.35} & \textbf{3.65} & \textbf{3.20} \\
\bottomrule
\end{tabular}

\section{Conclusion}
We introduced I-SplineFlow, which parameterizes the stochastic interpolant
scheduler with integrated monotone splines. The I-spline basis decouples the
polynomial degree from the number of control points and enforces boundary
conditions and a monotone SNR by construction, with no ordering constraint on
the parameters and closed-form derivatives for the transformed velocity.
Trained in minutes on a frozen pretrained model, it improves few-step FID over
B\'ezier-based scheduling in most settings, with the clearest gains at the
lowest NFEs. The gains are strongest on flow-based models, where I-SplineFlow
attains the best FID in every reported setting, and more mixed on diffusion
models. Closing that gap is the natural next step: we plan to explore other
spline bases that hold up better across diffusion schedules.
\clearpage

\bibliography{aaai2027}

\begin{thebibliography}{23}
\providecommand{\natexlab}[1]{#1}

\bibitem[{Albergo et~al.(2024)Albergo, Boffi, Lindsey, and
  Vanden-Eijnden}]{albergo2024multimarginal}
Albergo, M.~S.; Boffi, N.~M.; Lindsey, M.; and Vanden-Eijnden, E. 2024.
\newblock Multimarginal Generative Modeling with Stochastic Interpolants.
\newblock In \emph{ICLR}.

\bibitem[{Albergo, Boffi, and Vanden-Eijnden(2023)}]{albergo2023si}
Albergo, M.~S.; Boffi, N.~M.; and Vanden-Eijnden, E. 2023.
\newblock Stochastic Interpolants: A Unifying Framework for Flows and
  Diffusions.
\newblock \emph{Journal of Machine Learning Research}.

\bibitem[{Chen et~al.(2024)Chen, Zhou, Wang, Shen, and Lyu}]{chen2024gits}
Chen, D.; Zhou, Z.; Wang, C.; Shen, C.; and Lyu, S. 2024.
\newblock On the Trajectory Regularity of {ODE}-based Diffusion Sampling.
\newblock In \emph{ICML}.

\bibitem[{de~Boor(2001)}]{deboor2001splines}
de~Boor, C. 2001.
\newblock \emph{A Practical Guide to Splines}, volume~27 of \emph{Applied
  Mathematical Sciences}.
\newblock New York: Springer-Verlag, revised edition.

\bibitem[{Farouki and Rajan(1987)}]{farouki1987bernstein}
Farouki, R.~T.; and Rajan, V.~T. 1987.
\newblock On the numerical condition of polynomials in Bernstein form.
\newblock \emph{Computer Aided Geometric Design}, 4(3): 191--216.

\bibitem[{Ho, Jain, and Abbeel(2020)}]{ho2020ddpm}
Ho, J.; Jain, A.; and Abbeel, P. 2020.
\newblock Denoising Diffusion Probabilistic Models.
\newblock In \emph{NeurIPS}.

\bibitem[{Karras et~al.(2022)Karras, Aittala, Aila, and Laine}]{karras2022edm}
Karras, T.; Aittala, M.; Aila, T.; and Laine, S. 2022.
\newblock Elucidating the Design Space of Diffusion-Based Generative Models.
\newblock In \emph{NeurIPS}.

\bibitem[{Kim et~al.(2025)}]{kim2025simplereflow}
Kim, B.; et~al. 2025.
\newblock Simple {ReFlow}: Improved Techniques for Fast Flow Models.
\newblock In \emph{ICLR}.

\bibitem[{Kingma et~al.(2023)Kingma, Salimans, Poole, and Ho}]{kingma2023vdm}
Kingma, D.~P.; Salimans, T.; Poole, B.; and Ho, J. 2023.
\newblock Variational Diffusion Models.
\newblock In \emph{NeurIPS}.

\bibitem[{Lipman et~al.(2023)Lipman, Chen, Ben-Hamu, Nickel, and
  Le}]{lipman2023flow}
Lipman, Y.; Chen, R. T.~Q.; Ben-Hamu, H.; Nickel, M.; and Le, M. 2023.
\newblock Flow Matching for Generative Modeling.
\newblock In \emph{ICLR}.

\bibitem[{Liu, Gong, and Liu(2023)}]{liu2023reflow}
Liu, X.; Gong, C.; and Liu, Q. 2023.
\newblock Flow Straight and Fast: Learning to Generate and Transfer Data with
  Rectified Flow.
\newblock In \emph{ICLR}.

\bibitem[{Lu et~al.(2022)Lu, Zhou, Bao, Chen, Li, and Zhu}]{lu2022dpm}
Lu, C.; Zhou, Y.; Bao, F.; Chen, J.; Li, C.; and Zhu, J. 2022.
\newblock {DPM-Solver}: A Fast {ODE} Solver for Diffusion Probabilistic Model
  Sampling in Around 10 Steps.
\newblock In \emph{NeurIPS}.

\bibitem[{Min et~al.(2026)Min, Koo, Yoo, and Sung}]{min2025bezierflow}
Min, Y.; Koo, J.; Yoo, S.; and Sung, M. 2026.
\newblock {B\'ezierFlow}: Learning {B\'ezier} Stochastic Interpolant Schedulers
  for Few-Step Generation.
\newblock In \emph{ICLR}.

\bibitem[{Pokle et~al.(2024)Pokle, Muckley, Chen, and Karrer}]{pokle2024flows}
Pokle, A.; Muckley, M.~J.; Chen, R. T.~Q.; and Karrer, B. 2024.
\newblock Training-free Linear Image Inverses via Flows.
\newblock \emph{TMLR}.

\bibitem[{Shaul et~al.(2024)Shaul, Perez, Chen, Thabet, Pumarola, and
  Lipman}]{shaul2024bespoke}
Shaul, N.; Perez, J.; Chen, R. T.~Q.; Thabet, A.; Pumarola, A.; and Lipman, Y.
  2024.
\newblock Bespoke Solvers for Generative Flow Models.
\newblock In \emph{ICLR}.

\bibitem[{Song et~al.(2023)Song, Dhariwal, Chen, and
  Sutskever}]{song2023consistency}
Song, Y.; Dhariwal, P.; Chen, M.; and Sutskever, I. 2023.
\newblock Consistency Models.
\newblock In \emph{ICML}.

\bibitem[{Song et~al.(2021)Song, Sohl-Dickstein, Kingma, Kumar, Ermon, and
  Poole}]{song2021sde}
Song, Y.; Sohl-Dickstein, J.; Kingma, D.~P.; Kumar, A.; Ermon, S.; and Poole,
  B. 2021.
\newblock Score-Based Generative Modeling through Stochastic Differential
  Equations.
\newblock In \emph{ICLR}.

\bibitem[{Tong et~al.(2025)Tong, Hoang, Liu, Van~den Broeck, and
  Niepert}]{tong2025ld3}
Tong, V.; Hoang, T.-D.; Liu, A.; Van~den Broeck, G.; and Niepert, M. 2025.
\newblock Learning to Discretize Denoising Diffusion {ODE}s.
\newblock In \emph{ICLR}.

\bibitem[{Wang et~al.(2024)Wang, Li, Song, Li, Ge, Zheng, and
  Wang}]{wang2024flowdcn}
Wang, S.; Li, Z.; Song, T.; Li, X.; Ge, T.; Zheng, B.; and Wang, L. 2024.
\newblock {FlowDCN}: Exploring {DCN}-like Architectures for Fast Image
  Generation with Arbitrary Resolution.
\newblock In \emph{NeurIPS}.

\bibitem[{Xue et~al.(2024)Xue, Liu, Chen, Zhang, Hu, Xie, and Li}]{xue2024dmn}
Xue, S.; Liu, Z.; Chen, F.; Zhang, S.; Hu, T.; Xie, E.; and Li, Z. 2024.
\newblock Accelerating Diffusion Sampling with Optimized Time Steps.
\newblock In \emph{CVPR}.

\bibitem[{Zhang and Chen(2023)}]{zhang2023ipndm}
Zhang, Q.; and Chen, Y. 2023.
\newblock Fast Sampling of Diffusion Models with Exponential Integrator.
\newblock In \emph{ICLR}.

\bibitem[{Zhao et~al.(2026)Zhao, Lei, Yuan, Yang, Song, Wang, Zhu, and
  Zhang}]{zhao2026dyweight}
Zhao, T.; Lei, M.; Yuan, L.; Yang, Y.; Song, C.; Wang, Y.; Zhu, B.; and Zhang,
  C. 2026.
\newblock {DyWeight}: Dynamic Gradient Weighting for Few-Step Diffusion
  Sampling.
\newblock In \emph{CVPR}.

\bibitem[{Zhao et~al.(2023)Zhao, Bai, Rao, Zhou, and Lu}]{zhao2023unipc}
Zhao, W.; Bai, L.; Rao, Y.; Zhou, J.; and Lu, J. 2023.
\newblock {UniPC}: A Unified Predictor-Corrector Framework for Fast Sampling of
  Diffusion Models.
\newblock In \emph{NeurIPS}.

\end{thebibliography}

\clearpage
\appendix
\section*{Supplementary}




\section{Notation and Standing Conventions}
\label{sec:notation}

Throughout, $u=(u_0,\dots,u_{K+p})$ is a clamped knot vector on $[0,1]$ with
M-spline degree $p$ and $m$ \emph{simple} interior knots, so that
$u_0=\cdots=u_p=0$, $u_{K}=\cdots=u_{K+p}=1$, and
\begin{equation}
K=p+1+m,\qquad h_i:=u_{i+p+1}-u_i .
\label{eq:knots}
\end{equation}
Here $N_{i,p}$ is the Cox--de Boor B-spline basis, and
\begin{equation}
M_{i,p}=\frac{p+1}{h_i}\,N_{i,p},
\qquad
I_i(s)=\int_0^{s}M_{i,p}(v)\,dv ,
\label{eq:MI}
\end{equation}
so $\int_0^1 M_{i,p}=1$, $I_i(0)=0$ and $I_i(1)=1$. The scheduler coefficients of
the main paper are
\begin{equation}
\bar\alpha_\theta=\sum_{i=1}^{K}w^{(\alpha)}_iI_i,\qquad
\bar\sigma_\theta=1-\sum_{i=1}^{K}w^{(\sigma)}_iI_i,
\label{eq:coeffs-supp}
\end{equation}
with $w^{(\cdot)}=\mathrm{softmax}(\theta^{(\cdot)})$. We write $\rho=\alpha/\sigma$
for the frozen source SNR, $\bar\rho=\bar\alpha/\bar\sigma$ for the target SNR, and
$t_s=\rho^{-1}(\bar\rho(s))$ for the time map. A dot denotes $d/ds$ and a prime
$d/dt$, except on $\bar\rho$, where $\bar\rho'=d\bar\rho/ds$.

We use one elementary fact repeatedly.

\begin{lemma}[Interior positivity]
\label{lem:positivity}
For every $s\in(0,1)$ there exists $i$ with $M_{i,p}(s)>0$.
\end{lemma}

\begin{proof}
The B-spline basis is a partition of unity on $[u_p,u_K]=[0,1]$, so
$\sum_iN_{i,p}(s)=1$ and some $N_{i,p}(s)>0$. Since $h_i>0$, the function
$M_{i,p}=\frac{p+1}{h_i}N_{i,p}$ is a positive multiple of it.
\end{proof}

\section{Proofs of Theorems and Propositions}
\label{sec:proofs}

\subsection{Proof of Theorem 1 (Admissibility)}
For completeness we restate the claim as it appears in the main paper.

\setcounter{theorem}{0}
\begin{theorem}[Admissibility of the I-Spline scheduler]
\label{thm:admissible-supp}
Let $\{I_i\}_{i=1}^{K}$ be the order-$p$ I-spline basis of \eqref{eq:MI} on a
clamped knot vector on $[0,1]$ with simple interior knots, and let
$\bar\alpha_\theta,\bar\sigma_\theta$ be given by \eqref{eq:coeffs-supp} with
$w^{(\cdot)}=\mathrm{softmax}(\theta^{(\cdot)})$. Then for every
$\theta^{(\alpha)},\theta^{(\sigma)}\in\mathbb{R}^{K}$:
\begin{enumerate}
\item[(i)] $\bar\alpha(0){=}0$, $\bar\alpha(1){=}1$, $\bar\sigma(0){=}1$,
$\bar\sigma(1){=}0$;
\item[(ii)] the SNR $\bar\rho=\bar\alpha/\bar\sigma$ is strictly increasing on
$[0,1)$, so $\bar\rho^{-1}$ exists;
\item[(iii)] for any degree $p\ge2$, $\bar\alpha,\bar\sigma\in C^{p}([0,1])$, hence
in particular $C^{2}$.
\end{enumerate}
\end{theorem}

\begin{proof}
\emph{(i) Boundary conditions.} The softmax weights satisfy $w_i>0$ and
$\sum_iw_i=1$. Using $I_i(0)=0$ and $I_i(1)=1$,
\begin{align}
\bar\alpha(0) &= \sum_i w_i^{(\alpha)} I_i(0) = 0,\\
\bar\alpha(1) &= \sum_i w_i^{(\alpha)} I_i(1) = \sum_i w_i^{(\alpha)} = 1,
\end{align}
and likewise $\bar\sigma(0)=1-0=1$ and $\bar\sigma(1)=1-\sum_iw_i^{(\sigma)}=0$.

\emph{(ii) Strictly monotone SNR.} Differentiating \eqref{eq:coeffs-supp} and using
$I_i'=M_{i,p}$,
\begin{equation}
\dot{\bar\alpha}=\sum_i w_i^{(\alpha)} M_{i,p},\qquad
\dot{\bar\sigma}=-\sum_i w_i^{(\sigma)} M_{i,p}.
\label{eq:derivs}
\end{equation}
Fix $s\in(0,1)$. Every weight is strictly positive, every $M_{i,p}(s)\ge0$, and by
Lemma~\ref{lem:positivity} at least one $M_{i,p}(s)>0$; hence
$\dot{\bar\alpha}(s)>0$ and $\dot{\bar\sigma}(s)<0$. Therefore $\bar\alpha$
increases strictly from $\bar\alpha(0)=0$ and $\bar\sigma$ decreases strictly from
$\bar\sigma(0)=1$, so $\bar\alpha(s)\ge0$ and $\bar\sigma(s)>0$ for $s\in[0,1)$.
Writing $\bar\rho=\bar\alpha/\bar\sigma$,
\begin{equation}
\dot{\bar\rho}
=\frac{\dot{\bar\alpha}\,\bar\sigma-\bar\alpha\,\dot{\bar\sigma}}{\bar\sigma^{2}}.
\end{equation}
On $(0,1)$ the term $\dot{\bar\alpha}\,\bar\sigma$ is a product of two positive
quantities and the term $-\bar\alpha\,\dot{\bar\sigma}$ is non-negative, because
$\bar\alpha\ge0$ and $\dot{\bar\sigma}<0$. The numerator is therefore strictly
positive and the denominator is positive, so $\dot{\bar\rho}(s)>0$. By continuity
at $s=0$, $\bar\rho$ is strictly increasing on $[0,1)$, hence injective, and
$\bar\rho^{-1}$ exists on its range.

\emph{(iii) Smoothness, general $p$.} On a knot vector with simple interior knots,
$N_{i,p}$ is a piecewise polynomial of degree $p$ that is $C^{p-1}$ at each
interior knot and $C^{\infty}$ elsewhere; $M_{i,p}$ is a positive constant multiple
of $N_{i,p}$ and inherits this. Integration raises smoothness by exactly one order,
so $I_i\in C^{p}([0,1])$: it is $C^{p}$ at interior knots and piecewise polynomial
of degree $p{+}1$ between them. Near $s=0$ and $s=1$ each $I_i$ agrees with a
single polynomial, so no smoothness is lost at the clamped ends. A finite linear
combination of $C^{p}$ functions is $C^{p}$, hence
$\bar\alpha,\bar\sigma\in C^{p}([0,1])$, and $C^{p}\subseteq C^{2}$ for $p\ge2$.
\end{proof}

Requirements (i)--(iii) hold for arbitrary real $\theta$, so no ordering constraint
on the parameters is needed; the constraint is carried by the basis rather than
enforced during optimization.

\subsection{Proof of Proposition 1 (B\'ezier limit)}

\setcounter{proposition}{0}
\begin{proposition}[B\'ezier as the no-interior-knot limit]
\label{prop:bezier}
Fix $K$ mixture weights, and let the knot vector have no interior knots, so that
its only knots are $0$ and $1$, each with multiplicity $K$. Then the derivative
basis $\{M_i\}$ equals the degree-$(K{-}1)$ Bernstein basis
$\{b_{i,K-1}\}_{i=0}^{K-1}$ up to positive scaling, so $\dot{\bar\alpha}_\theta$ is
a non-negative combination of Bernstein polynomials and $\bar\alpha_\theta$ is a
monotone degree-$K$ Bernstein (B\'ezier-form) curve. In this $p\to K{-}1$ limit
every basis function is supported on all of $(0,1)$, so locality is lost; the
schedule is a globally supported monotone Bernstein form and need not coincide with
a nominal $K$-control-point B\'ezier scheduler, whose degree would be $K{-}1$
rather than $K$.
\end{proposition}

\begin{proof}
With $m=0$, \eqref{eq:knots} forces $p=K-1$ and the knot vector is
$u=(\underbrace{0,\dots,0}_{K},\underbrace{1,\dots,1}_{K})$.

\emph{Step 1 (basis).} On this knot vector the Cox--de Boor recursion collapses to
the Bernstein basis,
\begin{equation}
N_{i,K-1}(s)=\binom{K-1}{i}(1-s)^{K-1-i}s^{i}=b_{i,K-1}(s),
\end{equation}
for $i=0,\dots,K-1$, which is $K$ functions as required.

\emph{Step 2 (normalization).} For every $i$ we have $u_i=0$ and
$u_{i+p+1}=u_{i+K}=1$, so $h_i=1$ and $M_{i,K-1}=K\,b_{i,K-1}$: the derivative
basis is the Bernstein basis up to the single positive factor $K$. Hence
$\dot{\bar\alpha}_\theta=K\sum_iw_ib_{i,K-1}\ge0$.

\emph{Step 3 (the integral has degree $K$).} Using
$\int_0^{s}b_{i,n}=\frac{1}{n+1}\sum_{j=i+1}^{n+1}b_{j,n+1}(s)$ with $n=K-1$,
\begin{equation}
I_i(s)=K\!\int_0^{s}\!b_{i,K-1}=\sum_{j=i+1}^{K}b_{j,K}(s),
\end{equation}
so that, exchanging the order of summation,
\begin{equation}
\bar\alpha_\theta(s)=\sum_{j=0}^{K}C_j\,b_{j,K}(s),
\qquad
C_j=\sum_{i<j}w_i .
\label{eq:cumsum}
\end{equation}
The coefficients $C_j$ are nondecreasing with $C_0=0$ and $C_K=1$, so
$\bar\alpha_\theta$ is a monotone B\'ezier curve of degree $K$ with $K{+}1$ control
points, not degree $K{-}1$ with $K$ control points as a nominal
$K$-control-point B\'ezier scheduler would give. This is the sense in which the
limit is B\'ezier-\emph{form} rather than the B\'ezier baseline itself.

\emph{Step 4 (loss of locality).} Each $b_{i,K-1}$ is a polynomial vanishing only
at $s\in\{0,1\}$, hence nonzero on all of $(0,1)$: no basis function is compactly
supported inside the domain, which is the claimed loss of locality. This is the
$p\to K-1$ endpoint of the degree sweep in the main paper.
\end{proof}

\subsection{Proof of Proposition 2 (Degree/control-point decoupling)}

\begin{proposition}[Degree/control-point decoupling]
\label{prop:decouple-supp}
On a clamped knot vector on $[0,1]$ with M-spline degree $p$ and $m$ simple
interior knots, the number of mixture weights is $K=p+1+m$. Hence for any fixed $K$
every degree $1\le p\le K-1$ is realizable with $m=K-p-1\ge0$ interior knots, so
the degree varies independently of the weight count. The endpoint $m{=}0$ gives
$p{=}K{-}1$ and recovers Proposition~\ref{prop:bezier}.
\end{proposition}

\begin{proof}
A clamped knot vector of degree $p$ carries both endpoints at multiplicity $p+1$
and $m$ interior knots, for $2(p+1)+m$ knots in total. For a knot vector of length
$L$ the number of degree-$p$ B-spline basis functions is $L-p-1$, so
\begin{equation}
K=\big(2(p+1)+m\big)-p-1=p+1+m .
\end{equation}
Solving gives $m=K-p-1$, a non-negative integer exactly when $p\le K-1$. Each
admissible $p$ yields the same number $K$ of basis functions, hence the same number
of learnable weights, so the degree varies at fixed control-point count. The
endpoint $m=0$ gives $p=K-1$ and the knot vector of
Proposition~\ref{prop:bezier}.
\end{proof}

\begin{corollary}[Support width]
\label{cor:width}
The $m$ interior knots partition $[0,1]$ into $m+1$ knot spans. Each $M_{i,p}$ is
supported on $[u_i,u_{i+p+1})$, i.e.\ on at most $p+1$ consecutive spans, and at
most $p+1$ basis functions are nonzero at any $s$. With uniform interior knots the
support width as a fraction of the domain is
\begin{equation}
\omega(p,m)=\frac{p+1}{m+1}=\frac{p+1}{K-p},
\label{eq:width}
\end{equation}
which equals $1$ (global support) exactly at $m=0$.
\end{corollary}

Corollary~\ref{cor:width} is what makes degree and locality separable
experimentally: at fixed $K$ a sweep in $p$ moves $\omega$ as well, whereas varying
$(p,m)$ independently holds one fixed while moving the other.

\subsection{Proof of Proposition 3 (Monotonicity and valid reparameterization)}

\begin{proposition}[Monotonicity and valid reparameterization]
\label{prop:reparam-supp}
Let $\rho$ be the strictly increasing source SNR and $\bar\rho$ the target SNR,
with time map $t_s=\rho^{-1}(\bar\rho(s))$. Then
$dt_s/ds=\bar\rho'(s)/\rho'(t_s)$ has the sign of $\bar\rho'(s)$. If $\bar\rho$ is
not strictly increasing, then at points where $\bar\rho'(s)=0$ the source time
stalls and where $\bar\rho'(s)<0$ it reverses and $s\mapsto t_s$ fails to be
injective, so the sampling path is no longer a valid orientation-preserving
reparameterization of the source path. Strict monotonicity of $\bar\rho$
(Theorem~\ref{thm:admissible-supp}) makes $s\mapsto t_s$ a strictly increasing
bijection.
\end{proposition}

\begin{proof}
\emph{(a) Sign.} Differentiating $\rho(t_s)=\bar\rho(s)$ gives
$\rho'(t_s)\,dt_s/ds=\bar\rho'(s)$, and $\rho'>0$ by assumption, so
\begin{equation}
\frac{dt_s}{ds}=\frac{\bar\rho'(s)}{\rho'(t_s)},
\qquad
\mathrm{sign}\Big(\frac{dt_s}{ds}\Big)=\mathrm{sign}\big(\bar\rho'(s)\big).
\label{eq:dtds}
\end{equation}

\emph{(b) Stalling.} If $\bar\rho'(s^\ast)=0$ then $dt_s/ds$ vanishes at
$s^\ast$. The second term of the transformed velocity carries this factor, so the
update at $s^\ast$ reduces to the pure scaling $(\partial_s\log c_s)\bar x_s$: the
source state does not advance and $s\mapsto t_s$ is not a local diffeomorphism
there.

\emph{(c) Non-injectivity.} Suppose $\bar\rho'(s^\ast)<0$ for some
$s^\ast\in(0,1)$. We argue directly, without invoking a local extremum. Continuity
gives $\delta>0$ with $\bar\rho(s^\ast+\delta)<\bar\rho(s^\ast)$. On
$[s^\ast+\delta,1)$ the boundary conditions force
$\bar\rho\to\bar\alpha(1)/\bar\sigma(1)=+\infty$, so the supremum of $\bar\rho$
there exceeds $\bar\rho(s^\ast)$. By the intermediate value theorem there is
$s_2\in(s^\ast+\delta,1)$ with $\bar\rho(s_2)=\bar\rho(s^\ast)$. Setting
$s_1=s^\ast<s_2$ and using that $\rho^{-1}$ is single-valued,
\begin{equation}
t_{s_1}=\rho^{-1}\big(\bar\rho(s_1)\big)=\rho^{-1}\big(\bar\rho(s_2)\big)=t_{s_2},
\end{equation}
so distinct target times collapse onto one source time and $s\mapsto t_s$ is not
injective. The same conclusion holds if $\bar\rho$ is merely constant on a
subinterval, since then two distinct $s$ already share a value; no strict extremum
is required.

\emph{(d) Converse.} Under Theorem~\ref{thm:admissible-supp}(ii) we have $\bar\rho'>0$
on $(0,1)$, so by \eqref{eq:dtds} $dt_s/ds>0$ and $s\mapsto t_s$ is strictly
increasing, hence a bijection onto its range and orientation preserving.
\end{proof}

\section{Derivations}
\label{sec:derivations}

\subsection{The Path Transformation}

Matching the interpolant on both paths under $\bar x_s=c_s\,x_{t_s}$ gives
$c_s\alpha(t_s)=\bar\alpha(s)$ and $c_s\sigma(t_s)=\bar\sigma(s)$. Dividing the two
eliminates $c_s$ and yields the SNR-matching condition; the second equation then
returns the scale:
\begin{equation}
\rho(t_s)=\bar\rho(s),\quad
t_s=\rho^{-1}\big(\bar\rho(s)\big),\quad
c_s=\frac{\bar\sigma(s)}{\sigma(t_s)} .
\label{eq:cs}
\end{equation}
The SNR-matching condition is therefore not an assumption but the unique
consequence of requiring a \emph{scalar} reparameterization: any $c_s$ matching the
signal and noise components simultaneously must leave their ratio invariant.

Differentiating \eqref{eq:cs} and using \eqref{eq:dtds},
\begin{align}
\bar\rho'(s)&=\frac{\dot{\bar\alpha}\bar\sigma-\bar\alpha\dot{\bar\sigma}}
{\bar\sigma^{2}},
\qquad
\rho'(t)=\frac{\alpha'\sigma-\alpha\sigma'}{\sigma^{2}},
\label{eq:rhodots}\\
\frac{dt_s}{ds}&=\frac{\bar\rho'(s)}{\rho'(t_s)},
\label{eq:dtds2}\\
\partial_s\log c_s&=\frac{\dot{\bar\sigma}(s)}{\bar\sigma(s)}
-\frac{\sigma'(t_s)}{\sigma(t_s)}\,\frac{dt_s}{ds}.
\label{eq:dlogc}
\end{align}
Equivalently, in the unlogged form used in the implementation,
\begin{equation}
\frac{dc_s}{ds}
=\frac{\sigma(t_s)\,\dot{\bar\sigma}(s)
-\bar\sigma(s)\,\sigma'(t_s)\,\frac{dt_s}{ds}}{\sigma(t_s)^{2}} .
\label{eq:dcs}
\end{equation}

\paragraph{Transformed velocity.}
Differentiating $\bar x_s=c_sx_{t_s}$ with respect to $s$, then substituting
$x_{t_s}=\bar x_s/c_s$ and $dx_t/dt=u_t(x_t)$,
\begin{align}
\frac{d\bar x_s}{ds}
&=\dot c_s\,x_{t_s}+c_s\,\frac{dx_t}{dt}\Big|_{t_s}\frac{dt_s}{ds}
\nonumber\\
&=\big(\partial_s\log c_s\big)\bar x_s
+c_s\frac{dt_s}{ds}\,u_{t_s}\!\Big(\frac{\bar x_s}{c_s}\Big),
\label{eq:vel-supp}
\end{align}
which is the transformed velocity of the main paper. All four scheduler quantities
entering \eqref{eq:rhodots}--\eqref{eq:vel-supp} are available in closed form from
\eqref{eq:coeffs-supp} and \eqref{eq:derivs}; no separately learned derivative is
introduced.

\begin{remark}[Log-SNR form]
The implementation matches $\bar\lambda(s)=\log\bar\alpha-\log\bar\sigma$ against
$\lambda(t)=\log\alpha-\log\sigma$ and inverts $\lambda$ rather than $\rho$. Since
$\rho=e^{\lambda}$ is a strictly increasing bijection, $\rho(t_s)=\bar\rho(s)$ and
$\lambda(t_s)=\bar\lambda(s)$ define the same time map; the log form is preferred
numerically because $\rho$ spans many orders of magnitude on EDM schedules.
\end{remark}

\subsection{Closed-Form Inversion of the Source SNR}

The time map requires $\rho^{-1}$, which is analytic for both model families used:
\begin{align}
\text{ReFlow }(\alpha{=}t,\ \sigma{=}1{-}t):\ \
&\rho(t)=\frac{t}{1-t}
&&\Rightarrow\ t=\frac{\rho}{1+\rho},
\label{eq:invrf}\\
\text{EDM/VE }(\alpha{\equiv}1,\ \sigma{=}t):\ \
&\rho(t)=\frac{1}{t}
&&\Rightarrow\ t=\frac{1}{\rho}.
\label{eq:invve}
\end{align}
No numerical root find enters the training loop for either family, so
$\partial t_s/\partial\theta$ is exact rather than the derivative of an iterate. For
a VP source, $\rho^{-1}$ has no elementary closed form and is obtained by bisection
or Newton on $\lambda$; in that case the implicit function theorem supplies
$\partial t_s/\partial\bar\rho=1/\rho'(t_s)$ directly, so the iterate need not be
differentiated through.

\subsection{The Parameter Jacobian}

With $w=\mathrm{softmax}(\theta)$ we have
$\partial w_i/\partial\theta_j=w_i(\delta_{ij}-w_j)$, so
\begin{align}
\frac{\partial\bar\alpha(s)}{\partial\theta^{(\alpha)}_j}
&=\sum_iI_i(s)\,w_i(\delta_{ij}-w_j)
=w_j\big(I_j(s)-\bar\alpha(s)\big),
\label{eq:Ja}\\
\frac{\partial\dot{\bar\alpha}(s)}{\partial\theta^{(\alpha)}_j}
&=w_j\big(M_{j,p}(s)-\dot{\bar\alpha}(s)\big),
\label{eq:Jad}\\
\frac{\partial\bar\sigma(s)}{\partial\theta^{(\sigma)}_j}
&=-w_j\big(I_j(s)-(1-\bar\sigma(s))\big),
\label{eq:Js}\\
\frac{\partial\dot{\bar\sigma}(s)}{\partial\theta^{(\sigma)}_j}
&=-w_j\big(M_{j,p}(s)+\dot{\bar\sigma}(s)\big).
\label{eq:Jsd}
\end{align}
Each Jacobian is a \emph{centered} basis evaluation reweighted by $w_j$: the
parameter $\theta_j$ moves the schedule by the amount that its own basis function
deviates from the current curve.

\begin{corollary}[Exact rank deficiency]
\label{cor:rank}
Let $J\in\mathbb{R}^{G\times K}$ with
$J_{gj}=\partial\bar\alpha(s_g)/\partial\theta_j$ on any grid $\{s_g\}$. Then
$J\mathbf{1}=0$ exactly, so $\mathrm{rank}(J)\le K-1$ and $\sigma_K(J)=0$.
\end{corollary}

\begin{proof}
$\sum_jJ_{gj}=\sum_jw_jI_j(s_g)-\bar\alpha(s_g)\sum_jw_j
=\bar\alpha(s_g)-\bar\alpha(s_g)=0$.
\end{proof}

Corollary~\ref{cor:rank} is the reason the effective condition number is defined as
$\kappa_{\mathrm{eff}}=\sigma_1/\sigma_{K-1}$: the softmax is invariant to a
uniform shift of $\theta$, so one singular value vanishes identically and
$\sigma_1/\sigma_K$ would be infinite for \emph{any} basis, carrying no information
about locality.

\subsection{The Backward Pass}

Let the student solver produce $\bar x_{i+1}=\Phi_i(\bar x_i;\bar u_\theta)$ for
$i=0,\dots,M-1$ from shared noise $\bar x_0=x_0$, and let
$\mathcal{L}=d(\xi,\bar x_M)$. The gradient decomposes as
\begin{align}
\bar a_M&=\nabla_{\bar x_M}d,
\qquad
\bar a_i=\Big(\frac{\partial\Phi_i}{\partial\bar x_i}\Big)^{\!\top}\bar a_{i+1},
\label{eq:adjoint}\\
\nabla_\theta\mathcal{L}
&=\sum_{i=0}^{M-1}\Big(\frac{\partial\Phi_i}{\partial\theta}\Big)^{\!\top}
\bar a_{i+1},
\label{eq:gradtheta}
\end{align}
where $\partial\Phi_i/\partial\theta$ is the \emph{explicit} dependence at step
$i$. That dependence factors through four scalars per solver node,
$y=(\bar\alpha,\bar\sigma,\dot{\bar\alpha},\dot{\bar\sigma})$, along the chain
\begin{equation}
\theta\to w\to y\to\Big(t_s,\ c_s,\ \frac{dt_s}{ds}\Big)\to\bar u_s\to
\bar x_M\to\mathcal{L}.
\label{eq:chain}
\end{equation}

\paragraph{Node Jacobians.}
From \eqref{eq:cs}--\eqref{eq:dlogc},
\begin{align}
\frac{\partial t_s}{\partial\bar\alpha}
&=\frac{1}{\bar\sigma\,\rho'(t_s)},
\qquad
\frac{\partial t_s}{\partial\bar\sigma}
=\frac{-\bar\alpha}{\bar\sigma^{2}\,\rho'(t_s)},
\label{eq:dts}\\
\frac{\partial c_s}{\partial\bar\sigma}
&=\frac{1}{\sigma(t_s)}
-\frac{\bar\sigma\,\sigma'(t_s)}{\sigma(t_s)^{2}}
\frac{\partial t_s}{\partial\bar\sigma},
\label{eq:dcsds}\\
\frac{\partial c_s}{\partial\bar\alpha}
&=-\frac{\bar\sigma\,\sigma'(t_s)}{\sigma(t_s)^{2}}
\frac{\partial t_s}{\partial\bar\alpha},
\label{eq:dcsda}
\end{align}
and, with $\bar\rho'$ as in \eqref{eq:rhodots},
\begin{align}
\frac{\partial}{\partial y}\frac{dt_s}{ds}
&=\frac{1}{\rho'(t_s)}\frac{\partial\bar\rho'}{\partial y}
-\frac{\bar\rho'\,\rho''(t_s)}{\rho'(t_s)^{2}}
\frac{\partial t_s}{\partial y},
\label{eq:ddtds}\\
\frac{\partial\bar\rho'}{\partial\dot{\bar\alpha}}&=\frac{1}{\bar\sigma},
\qquad
\frac{\partial\bar\rho'}{\partial\dot{\bar\sigma}}
=\frac{-\bar\alpha}{\bar\sigma^{2}} .
\label{eq:drhop}
\end{align}
The remaining factor $\partial y/\partial\theta$ is \eqref{eq:Ja}--\eqref{eq:Jsd}.

\paragraph{Where the frozen network enters.}
Differentiating \eqref{eq:vel-supp} with $x_t=\bar x_s/c_s$,
\begin{align}
\frac{\partial\bar u_s}{\partial c_s}
&=\frac{dt_s}{ds}\Big[u_{t_s}(x_t)
-\frac{1}{c_s}\,\partial_xu_{t_s}(x_t)\,\bar x_s\Big],
\label{eq:dudc}\\
\frac{\partial\bar u_s}{\partial t_s}
&=c_s\,\frac{dt_s}{ds}\,\partial_tu_{t_s}(x_t).
\label{eq:dudt}
\end{align}
Both terms require differentiating $S_\phi$ with respect to its \emph{inputs} --
the state through $\partial_xu$ and the time embedding through $\partial_tu$ --
which is the precise sense in which gradients flow through the frozen network. No
term in \eqref{eq:adjoint}--\eqref{eq:dudt} involves $\partial/\partial\phi$: the
network weights receive no update, and $\phi$ may be held in inference mode with
gradients enabled only on the input arguments. Memory scales as $O(M)$ stored
activations, or $O(1)$ with recomputation at the cost of a second forward pass per
step.

\subsection{Exact Linear Initialization}

\begin{lemma}[The linear scheduler is exactly representable]
\label{lem:init}
Let $h_i=u_{i+p+1}-u_i$. Then $\sum_ih_i=p+1$, and the choice
\begin{equation}
w_i^{\star}=\frac{h_i}{p+1}
\qquad\Longleftrightarrow\qquad
\theta_i^{\star}=\log h_i+\mathrm{const}
\label{eq:init}
\end{equation}
satisfies $w^\star>0$ and $\sum_iw^\star_i=1$, and yields $\dot{\bar\alpha}\equiv1$,
$\bar\alpha(s)=s$ and $\bar\sigma(s)=1-s$, for every degree $p$ and every knot
placement.
\end{lemma}

\begin{proof}
From $M_{i,p}=\frac{p+1}{h_i}N_{i,p}$ and $\int_0^1M_{i,p}=1$ we get
$\int_0^1N_{i,p}=\frac{h_i}{p+1}$. Summing over $i$ and using the partition of
unity $\sum_iN_{i,p}\equiv1$ on $[0,1]$,
\begin{equation}
1=\int_0^1\sum_iN_{i,p}=\sum_i\frac{h_i}{p+1}
\quad\Longrightarrow\quad\sum_ih_i=p+1 .
\end{equation}
Hence $w^\star$ is a probability vector, and
\begin{equation}
\dot{\bar\alpha}=\sum_iw^\star_iM_{i,p}
=\sum_i\frac{h_i}{p+1}\cdot\frac{p+1}{h_i}N_{i,p}
=\sum_iN_{i,p}\equiv1 ,
\end{equation}
so $\bar\alpha(s)=\int_0^s1=s$ and, with the same weights for $\sigma$,
$\bar\sigma(s)=1-s$. Since the softmax is invariant to a uniform shift of $\theta$,
any additive constant in $\theta^\star_i=\log h_i$ gives the same $w^\star$.
\end{proof}

\begin{remark}
Lemma~\ref{lem:init} is exact, not approximate: the linear scheduler lies in the
model class for every $(p,m)$, so every configuration of the degree sweep starts
from an identical curve. Any residual $\max_s|\bar\alpha(s)-s|$ observed in practice
is attributable to the quadrature grid used to tabulate $I_i$, not to the weights.
\end{remark}

\subsection{The I-Spline Model is a Monotone B-Spline Model}

Let $\bar u$ be the knot vector $u$ with one additional copy of the right endpoint
appended, so that $\bar u$ has left multiplicity $p+1$, right multiplicity $p+2$,
and carries exactly $K$ B-spline basis functions $\{N_{j,p+1}\}_{j=1}^{K}$ of
degree $p+1$.

\begin{lemma}[Integrated basis]
\label{lem:ramsay}
$\displaystyle I_i(s)=\sum_{j\ge i}N_{j,p+1}(s)$ for $i=1,\dots,K$.
\end{lemma}

\begin{proof}
The standard integration identity for B-splines is
$\int_0^{s}N_{i,p}=\frac{h_i}{p+1}\sum_{j\ge i}N_{j,p+1}(s)$ on $\bar u$.
Multiplying by $(p+1)/h_i$ gives $I_i=\int_0^sM_{i,p}=\sum_{j\ge i}N_{j,p+1}$.
\end{proof}

\begin{proposition}[Reparameterization]
\label{prop:equiv}
With $c_j:=\sum_{i\le j}w_i$,
\begin{equation}
\bar\alpha_\theta(s)=\sum_iw_iI_i(s)=\sum_{j=1}^{K}c_j\,N_{j,p+1}(s),
\label{eq:equiv}
\end{equation}
where $0<c_1\le c_2\le\cdots\le c_K=1$. Hence the I-spline scheduler is exactly a
degree-$(p{+}1)$ B-spline curve whose control points are constrained to be
nondecreasing and to terminate at $1$; conversely, any such monotone control
polygon is realized by $w_i=c_i-c_{i-1}$.
\end{proposition}

\begin{proof}
Substitute Lemma~\ref{lem:ramsay} into \eqref{eq:coeffs-supp} and exchange the order of
summation; the coefficient of $N_{j,p+1}$ is $\sum_{i\le j}w_i=c_j$. Monotonicity
of $c$ follows from $w_i>0$, and $c_K=\sum_iw_i=1$.
\end{proof}

Proposition~\ref{prop:equiv} has three consequences used in the experiments.
(i)~A ``monotone B-spline'' control is not a missing baseline: it \emph{is} the
I-spline model in different coordinates. (ii)~The B-spline ablation therefore
differs from the I-spline arm in exactly two respects -- removal of the ordering
constraint $c_{j-1}\le c_j$, and replacement of the softmax geometry by free
coordinates -- and not in basis, span, or expressive family. (iii)~The comparison
is degree-matched only when the free-control-point arm uses degree $p{+}1$, since
$\bar\alpha_\theta$ is piecewise degree $p{+}1$ while $\sum_ic_iN_{i,p}$ is
piecewise degree $p$.

\subsection{Evaluating the Basis}

Lemma~\ref{lem:ramsay} gives an exact evaluation of $I_i$ as a degree-$(p{+}1)$
B-spline partial sum, with the boundary convention
\begin{equation}
I_i(s)=0\ \ (s\le u_i),
\qquad
I_i(s)=1\ \ (s\ge u_{i+p+1}),
\label{eq:Iconv}
\end{equation}
and it is the derivative $M_{i,p}$, not $I_i$, that is compactly supported. An
equivalent tabulation integrates $M_{i,p}$ on a uniform grid of $G$ points by the
trapezoidal rule and interpolates, at $O(G^{-2})$ accuracy; we use $G=512$, for
which the two agree to $2\times10^{-4}$. The closed form requires care with
indices: the sum in Lemma~\ref{lem:ramsay} runs over the degree-$(p{+}1)$ basis on
the \emph{right-extended} knot vector $\bar u$, and using the degree-$p$ basis on
$u$ instead violates $I_i(0)=0$.

\subsection{Diagnostics Used in the Experiments}

Let $J\in\mathbb{R}^{G\times K}$ be the Jacobian of Corollary~\ref{cor:rank} on a
uniform grid of $G=512$ points, with columns $J_j$ and singular values
$\sigma_1\ge\cdots\ge\sigma_K$. The effective condition number, parameter coupling
and coupling bandwidth are
\begin{align}
\kappa_{\mathrm{eff}}&=\frac{\sigma_1}{\sigma_{K-1}},
\label{eq:kappa}\\
C_{ij}&=\frac{|\langle J_i,J_j\rangle|}{\|J_i\|\,\|J_j\|},
\qquad
\beta=\frac{\sum_{ij}|i-j|\,C_{ij}}{\sum_{ij}C_{ij}},
\label{eq:band}
\end{align}
and the participation ratio, with $a_{gj}=|J_{gj}|$, is
\begin{equation}
\mathrm{PR}=\frac{1}{G}\sum_g\frac{1}{K}
\frac{\big(\sum_ja_{gj}^{2}\big)^{2}}{\sum_ja_{gj}^{4}} .
\label{eq:pr}
\end{equation}
$\mathrm{PR}\in(0,1]$ measures the fraction of parameters that meaningfully move
the curve at a given $s$: $\mathrm{PR}\to1/K$ indicates value-locality and
$\mathrm{PR}\to1$ a fully global response. The derivative-sensitivity map of the
locality figure is
$|\partial\dot{\bar\alpha}(s)/\partial\theta_i|=|w_i(M_{i,p}(s)-\dot{\bar\alpha}(s))|$
from \eqref{eq:Jad}, normalized per row.

For the monotonicity ablation we additionally log, per checkpoint, the violation
fraction $\nu$, the extremal derivative and noise level, and the clamp fraction:
\begin{equation}
\nu=\frac{\big|\{s:\bar\rho'(s)<0\}\big|}{G},
\quad
\min_s\bar\rho'(s),
\quad
\min_s\bar\sigma(s).
\label{eq:monodiag}
\end{equation}
The clamp fraction is the proportion of grid points at which the raw curve leaves
$[\epsilon,1]$ and is clipped. It matters because clipping zeroes the gradient at
those points, so a free-control-point arm with a nonzero clamp fraction differs
from the I-spline arm by more than monotonicity alone.

\section{Algorithms}
\label{sec:algorithms}

\begin{algorithm}[ht]
\caption{Basis construction}
\label{alg:basis}
\begin{algorithmic}[1]
\REQUIRE weight count $K$, degree $p$ (with $p\le K-1$), grid size $G$
\ENSURE knots $u$, matrices $N,M,I\in\mathbb{R}^{G\times K}$
\STATE $m\leftarrow K-p-1$ \COMMENT{Prop.~\ref{prop:decouple-supp}}
\IF{$m>0$}
  \STATE $\mathrm{int}\leftarrow m$ uniform points in $(0,1)$
  \STATE $u\leftarrow[0^{(p+1)},\ \mathrm{int},\ 1^{(p+1)}]$
\ELSE
  \STATE $u\leftarrow[0^{(p+1)},\ 1^{(p+1)}]$
  \COMMENT{Bernstein limit, Prop.~\ref{prop:bezier}}
\ENDIF
\STATE $s_g\leftarrow g/(G-1)$, $g=0,\dots,G-1$
\STATE $N\leftarrow$ Cox--de Boor degree-$p$ basis on $u$ at $\{s_g\}$
\STATE $h_i\leftarrow u_{i+p+1}-u_i$;\ \ $M_{gi}\leftarrow(p+1)N_{gi}/h_i$
\STATE $\bar u\leftarrow u$ with one extra copy of $1$ appended
\STATE $\bar N\leftarrow$ degree-$(p{+}1)$ basis on $\bar u$ at $\{s_g\}$
\STATE $I_{gi}\leftarrow\sum_{j\ge i}\bar N_{gj}$
\COMMENT{Lemma~\ref{lem:ramsay}; or trapezoid of $M$}
\STATE \textbf{return} $u,N,M,I$
\end{algorithmic}
\end{algorithm}

\begin{algorithm}[ht]
\caption{Linear initialization (exact)}
\label{alg:init}
\begin{algorithmic}[1]
\REQUIRE knots $u$, degree $p$, weight count $K$
\ENSURE $\theta^{(\alpha)},\theta^{(\sigma)}$ giving $\bar\alpha(s){=}s$,
$\bar\sigma(s){=}1{-}s$
\STATE $h_i\leftarrow u_{i+p+1}-u_i$
\COMMENT{$\sum_ih_i=p+1$, Lemma~\ref{lem:init}}
\STATE $\theta_i\leftarrow\log h_i$ \COMMENT{shift-invariant}
\STATE $\theta^{(\alpha)}\leftarrow\theta$;\ \ $\theta^{(\sigma)}\leftarrow\theta$
\STATE \textbf{assert} $\max_s|\bar\alpha_\theta(s)-s|<10^{-3}$
\COMMENT{quadrature check}
\STATE \textbf{return} $\theta^{(\alpha)},\theta^{(\sigma)}$
\end{algorithmic}
\end{algorithm}

\begin{algorithm}[ht]
\caption{Scheduler evaluation at a solver node}
\label{alg:eval}
\begin{algorithmic}[1]
\REQUIRE $\theta^{(\alpha)},\theta^{(\sigma)}$, bases $M,I$, node $s$, source
$(\alpha,\sigma)$
\ENSURE $t_s,\ c_s,\ dt_s/ds,\ \partial_s\log c_s$
\STATE $w^{(\cdot)}\leftarrow\mathrm{softmax}(\theta^{(\cdot)})$
\STATE $\bar\alpha\leftarrow\sum_iw^{(\alpha)}_iI_i(s)$;\ \
       $\dot{\bar\alpha}\leftarrow\sum_iw^{(\alpha)}_iM_i(s)$
\STATE $\bar\sigma\leftarrow1-\sum_iw^{(\sigma)}_iI_i(s)$;\ \
       $\dot{\bar\sigma}\leftarrow-\sum_iw^{(\sigma)}_iM_i(s)$
\STATE $\bar\rho\leftarrow\bar\alpha/\bar\sigma$
\STATE $\bar\rho'\leftarrow(\dot{\bar\alpha}\bar\sigma
       -\bar\alpha\dot{\bar\sigma})/\bar\sigma^{2}$
\STATE $t_s\leftarrow\rho^{-1}(\bar\rho)$
\COMMENT{closed form, \eqref{eq:invrf}/\eqref{eq:invve}}
\STATE $\rho'(t_s)\leftarrow
       (\alpha'\sigma-\alpha\sigma')/\sigma^{2}\big|_{t_s}$
\STATE $dt_s/ds\leftarrow\bar\rho'/\rho'(t_s)$;\ \
       $c_s\leftarrow\bar\sigma/\sigma(t_s)$
\STATE $\partial_s\log c_s\leftarrow\dot{\bar\sigma}/\bar\sigma
       -\big(\sigma'(t_s)/\sigma(t_s)\big)\,dt_s/ds$
\STATE \textbf{return} $t_s,c_s,dt_s/ds,\partial_s\log c_s$
\end{algorithmic}
\end{algorithm}

\begin{algorithm}[ht]
\caption{Training (teacher forcing, frozen $S_\phi$)}
\label{alg:train}
\begin{algorithmic}[1]
\REQUIRE frozen $S_\phi$; teacher grid $\{t_i\}_{i=1}^{N}$; student grid
$\{s_i\}_{i=1}^{M}$, $M\!\ll\!N$; noise set $\mathcal{X}_0$; epochs $E$; distance
$d$
\ENSURE trained $\theta=(\theta^{(\alpha)},\theta^{(\sigma)})$
\STATE build bases (Alg.~\ref{alg:basis}); init $\theta$ (Alg.~\ref{alg:init})
\STATE cache teacher targets $\xi(x_0)$ for $x_0\in\mathcal{X}_0$
\COMMENT{once, no grad}
\FOR{$e=1$ \TO $E$}
  \FOR{each minibatch $\{x_0\}$}
    \STATE $\bar x\leftarrow x_0$
    \FOR{$i=1$ \TO $M$}
      \STATE get $(t_s,c_s,dt_s/ds,\partial_s\log c_s)$ at $s_i$
      \COMMENT{Alg.~\ref{alg:eval}}
      \STATE $u\leftarrow S_\phi(\bar x/c_s,\ t_s)$
      \COMMENT{frozen; grad w.r.t.\ inputs only}
      \STATE $\bar u\leftarrow(\partial_s\log c_s)\bar x
             +c_s(dt_s/ds)\,u$
      \STATE $\bar x\leftarrow\Phi_i(\bar x,\bar u)$
      \COMMENT{RK1/RK2/UniPC/iPNDM}
    \ENDFOR
    \STATE $\mathcal{L}\leftarrow d\big(\xi(x_0),\bar x\big)$
    \STATE $\theta\leftarrow\mathrm{RMSprop}(\theta,\nabla_\theta\mathcal{L})$
    \COMMENT{\eqref{eq:adjoint}--\eqref{eq:dudt}; $\phi$ never updated}
  \ENDFOR
\ENDFOR
\STATE \textbf{return} $\theta$
\end{algorithmic}
\end{algorithm}

\begin{algorithm}[ht]
\caption{Sampling under a learned scheduler}
\label{alg:sample}
\begin{algorithmic}[1]
\REQUIRE trained $\theta$, frozen $S_\phi$, nodes $\{s_i\}_{i=1}^{M}$, base solver
$\Phi$
\ENSURE sample $\bar x$
\STATE $\bar x\leftarrow x_0\sim p_0$, scaled by the source prior transformation
\FOR{$i=1$ \TO $M$}
  \STATE get $(t_s,c_s,dt_s/ds,\partial_s\log c_s)$ at $s_i$
  \COMMENT{Alg.~\ref{alg:eval}}
  \STATE $\bar u\leftarrow(\partial_s\log c_s)\bar x
         +c_s(dt_s/ds)\,S_\phi(\bar x/c_s,\ t_s)$
  \STATE $\bar x\leftarrow\Phi(\bar x,\bar u,s_i,s_{i+1})$
\ENDFOR
\STATE \textbf{return} $\bar x$
\end{algorithmic}
\end{algorithm}

\begin{algorithm}[ht]
\caption{Non-monotone B-spline ablation}
\label{alg:bspline}
\begin{algorithmic}[1]
\REQUIRE knots $u$, degree $q$ ($q{=}p{+}1$ for a degree-matched comparison,
Prop.~\ref{prop:equiv})
\ENSURE non-monotone scheduler at matched basis, knots and budget
\STATE free control points $c^{(\alpha)},c^{(\sigma)}$ with pinned endpoints
       $c^{(\alpha)}_1{=}0$, $c^{(\alpha)}_K{=}1$,
       $c^{(\sigma)}_1{=}1$, $c^{(\sigma)}_K{=}0$
\STATE $\bar\alpha(s)\leftarrow\sum_jc^{(\alpha)}_jN_{j,q}(s)$
\COMMENT{no softmax, no ordering}
\STATE $\dot{\bar\alpha}(s)\leftarrow\sum_{j\ge2}
       \frac{q(c^{(\alpha)}_j-c^{(\alpha)}_{j-1})}{u_{j+q}-u_j}N_{j,q-1}(s)$
\COMMENT{analytic, as in the I-spline arm}
\STATE initialize $c_j$ to the linear schedule
\COMMENT{Alg.~\ref{alg:init} via Prop.~\ref{prop:equiv}}
\STATE clip $\bar\alpha,\bar\sigma$ to $[\epsilon,1]$; record clamp fraction
\STATE log $\nu$, $\min_s\bar\rho'$, $\min_s\bar\sigma$ each step
\COMMENT{\eqref{eq:monodiag}}
\STATE \textbf{if} $\min_s\bar\rho'<0$ \textbf{then} the run has left the
       admissible set (Prop.~\ref{prop:reparam-supp}); the step is flagged, not
       discarded
\end{algorithmic}
\end{algorithm}

\begin{algorithm}[ht]
\caption{Conditioning and monotonicity diagnostic}
\label{alg:diag}
\begin{algorithmic}[1]
\REQUIRE checkpoint $\theta$, bases $M,I$ on $G$ grid points
\ENSURE $\kappa_{\mathrm{eff}},\beta,\mathrm{PR},\nu,
\min_s\bar\rho',\min_s\bar\sigma$
\STATE $w\leftarrow\mathrm{softmax}(\theta)$;\ \
       $\bar\alpha_g\leftarrow\sum_iw_iI_{gi}$
\STATE $J_{gj}\leftarrow w_j(I_{gj}-\bar\alpha_g)$
\COMMENT{\eqref{eq:Ja}; verify against autograd}
\STATE $\{\sigma_k\}\leftarrow\mathrm{SVD}(J)$;\ \
       $\kappa_{\mathrm{eff}}\leftarrow\sigma_1/\sigma_{K-1}$
\COMMENT{$\sigma_K=0$, Cor.~\ref{cor:rank}}
\STATE $C_{ij}\leftarrow|\langle J_i,J_j\rangle|/(\|J_i\|\|J_j\|)$
\STATE $\beta\leftarrow\sum_{ij}|i-j|C_{ij}\big/\sum_{ij}C_{ij}$
\STATE $\mathrm{PR}\leftarrow$ Eq.~\eqref{eq:pr}
\STATE $\bar\rho_g\leftarrow\bar\alpha_g/\bar\sigma_g$;\ \
       $\nu\leftarrow\frac{1}{G}|\{g:\Delta\bar\rho_g<0\}|$
\STATE \textbf{return} all diagnostics
\end{algorithmic}
\end{algorithm}

\FloatBarrier

\begin{figure*}[!t]
\centering
\includegraphics[width=0.995\textwidth]{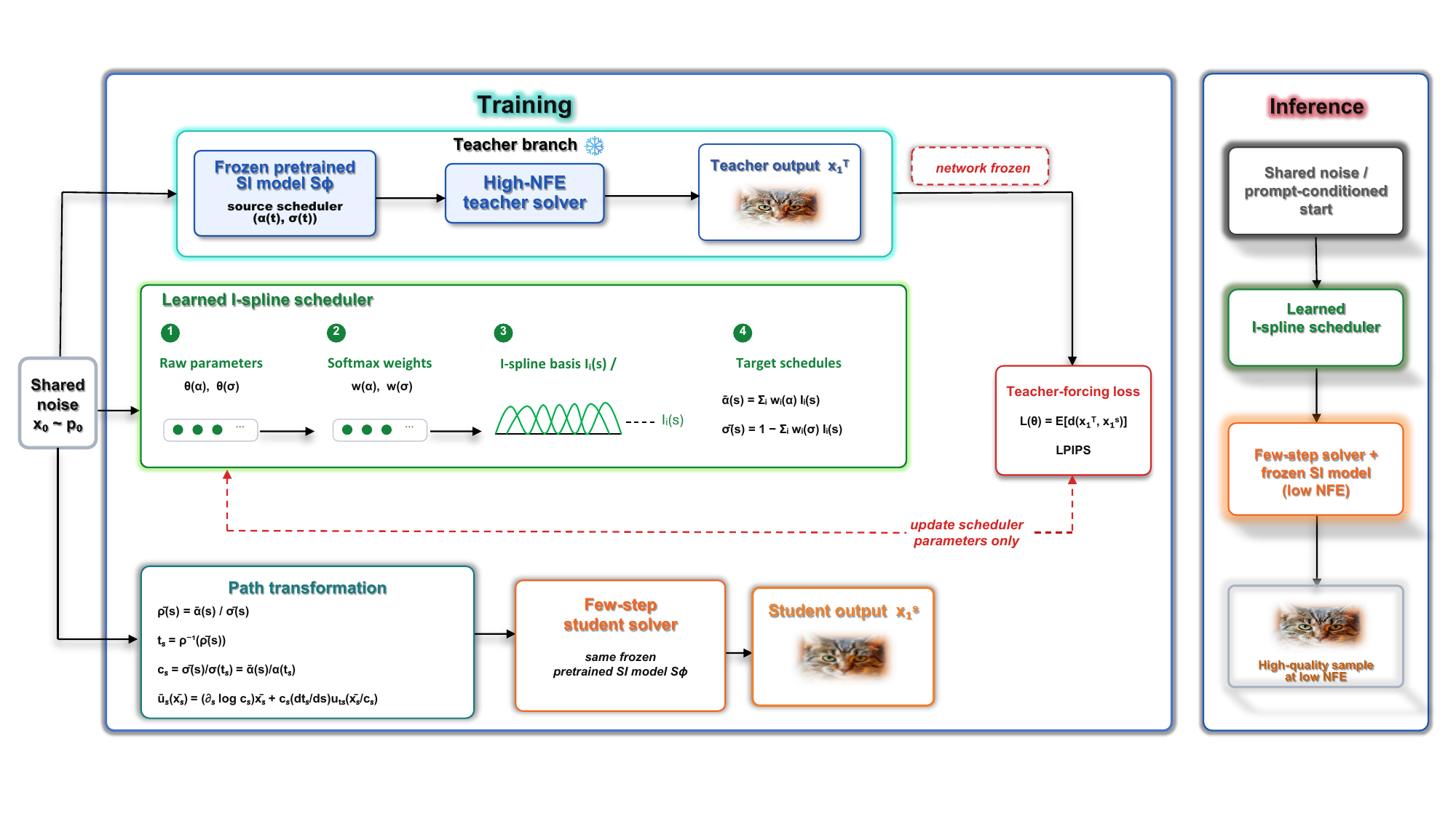}
\caption{Architecture Diagram of I-SplineFlow}
\label{fig:Diagram}

\vspace{0.1em}
\section{Training Settings}

\small
\renewcommand{\arraystretch}{1.00}
\captionof{table}{Hyperparameters for I-SplineFlow. Shared settings in the top block,
dataset-specific settings below.}
\label{tab:hyperparams}
\begin{tabular}{ll}
\toprule
\multicolumn{2}{c}{\textbf{Shared settings}} \\
\midrule
Control points $K$              & 32 \\
Spline degree $p$               & 3 (flow), 16 (diffusion) \\
Interior knots $m=K-p-1$        & 28 (flow), 15 (diffusion) \\
Knot placement                  & uniform on $[0,1]$ \\
Initialization                  & linear, $w_i=h_i/(p{+}1)$ (Lemma~\ref{lem:init}) \\
Training loss                   & LPIPS \\
Optimizer                       & RMSprop \\
Optimizer momentum              & 0.9 \\
Weight decay                    & 0.0 \\
LR decay factor                 & 0.8 \\
Patience                        & 5 \\
Min.\ LR (stage 1 / stage 2)    & $5\times10^{-5}$ / $1\times10^{-6}$ \\
Training rounds                 & 8 (5 for ImageNet) \\
Random seed                     & 0 \\
GPU                        & 2$\times$ NVIDIA T4 \\
\bottomrule
\end{tabular}

\vspace{0.35em}

\begin{tabular}{lcccc}
\toprule
\textbf{Setting} & \textbf{CIFAR-10} & \textbf{FFHQ} & \textbf{AFHQv2} & \textbf{ImageNet} \\
 & (32$\times$32) & (64$\times$64) & (64$\times$64) & (256$\times$256) \\
\midrule
Base model               & EDM         & EDM         & EDM         & FlowDCN-XL \\
Base schedule            & VE          & VE          & VE          & RF \\
Train batch size         & 8           & 2           & 2           & 2 \\
Valid batch size         & 20          & 10          & 10          & 10 \\
\# train samples         & 200         & 50          & 50          & 50 \\
\# valid samples         & 200         & 50          & 50          & 50 \\
LR (stage 1)             & 0.005       & 0.005       & 0.005       & 0.005 \\
LR (stage 2)             & 0.03        & 0.01        & 0.01        & 0.1 \\
CFG scale                & ---         & ---         & ---         & 4.0 \\
\bottomrule
\end{tabular}
\end{figure*}

\begin{figure*}[!t]
\section{Additional Ablation}
\subsection{Degree-Control Point Decoupling}

\centering
\small
\setlength{\tabcolsep}{4pt}
\captionof{table}{Degree/locality grid on CIFAR-10 with EDM and UniPC. FID against the
EDM reference statistics, lower is better; best per column in \textbf{bold}.
$m=K-p-1$ is the interior knot count and $\omega=(p{+}1)/(m{+}1)$ the support
width; $m{=}0$ is the global limit (Prop.~\ref{prop:bezier}). \textbf{Contribution 1:}
B\'ezierFlow is stuck at $p=K-1$, so it fills only the $m{=}0$ row of each
block. Degree matters as much as weight count, since at $K{=}16$ and
NFE${=}4$, $p{=}3$ gives $8.97$ against $12.76$ for $p{=}15$ at the same
parameter count, and the degree we choose also beats the one $K$ forces on
B\'ezierFlow ($8.97$ vs.\ $9.95$). \textbf{Contribution 3:} the $m{=}0$ rows are never the
best I-spline setting, and the best cell is at an interior degree ($K{=}32$,
$p{=}16$), matching the conditioning and expressivity trade-off.}
\label{tab:grid}
\begin{tabular}{l c c c c cccc}
\toprule
& & & & & \multicolumn{4}{c}{NFE} \\
\cmidrule(lr){6-9}
Method & $K$ & $p$ & $m$ & $\omega$ & 4 & 6 & 8 & 10 \\
\midrule
UniPC (base) & --- & --- & --- & --- & 50.55 & 19.59 & 10.02 & 6.48 \\
\midrule
\multicolumn{9}{@{}l}{\textit{$K=16$}} \\
+ B\'ezierFlow & 16 & 15 & 0  & 1.00 & 9.95  & 3.43 & 2.55 & \textbf{2.32} \\
+ I-Spline     & 16 & 3  & 12 & 0.31 & 8.97  & 3.70 & 2.73 & 2.45 \\
+ I-Spline     & 16 & 15 & 0  & 1.00 & 12.76 & 4.00 & 3.15 & 2.75 \\
\midrule
\multicolumn{9}{@{}l}{\textit{$K=32$}} \\
+ B\'ezierFlow & 32 & 31 & 0  & 1.00 & 9.81 & 3.39 & 2.79 & 2.40 \\
+ I-Spline     & 32 & 3  & 28 & 0.14 & 9.01 & 3.69 & 2.90 & 2.63 \\
+ I-Spline     & 32 & 16 & 15 & 1.06 & \textbf{8.87} & \textbf{3.27} & \textbf{2.44} & 2.38 \\
+ I-Spline     & 32 & 31 & 0  & 1.00 & 11.22 & 3.68 & 3.09 & 2.62 \\
\midrule
\multicolumn{9}{@{}l}{\textit{$K=64$}} \\
+ B\'ezierFlow & 64 & 63 & 0  & 1.00 & 8.78  & 3.15 & 3.31 & 2.45 \\
+ I-Spline     & 64 & 3  & 60 & 0.07 & 9.01  & 3.25 & 2.79 & 2.65 \\
+ I-Spline     & 64 & 16 & 47 & 0.35 & 10.28 & 4.00 & 2.92 & 2.57 \\
+ I-Spline     & 64 & 63 & 0  & 1.00 & 8.78  & 3.15 & 3.31 & 2.45 \\
\bottomrule
\end{tabular}

\vspace{0.5em}
\includegraphics[width=\textwidth]{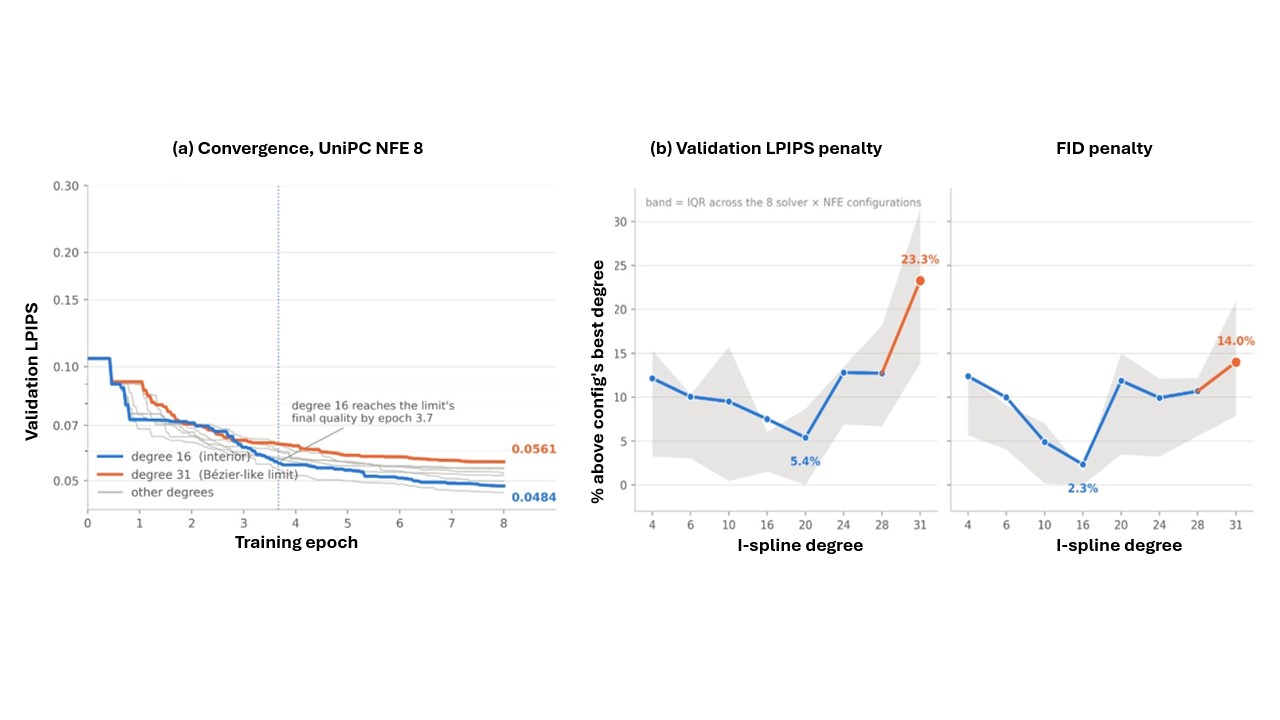}
\captionof{figure}{Quality degrades as the degree approaches the B\'ezier-like limit,
supporting \textbf{Contribution~3} (and \textbf{Contribution~1}, since only the I-spline can
reach an interior degree). (a) Convergence at UniPC NFE${=}8$: the interior
degree $16$ reaches the degree-$31$ curve's final quality by epoch $3.7$ and
settles below it ($0.0484$ vs.\ $0.0561$). (b) Validation LPIPS and (c) FID
penalty, in percent above each configuration's best degree, averaged over all
eight solver$\times$NFE settings (band: IQR). The FID penalty is smallest at
$p{=}16$ ($2.3\%$) and the LPIPS penalty near $p{=}20$ ($5.4\%$), both rising
to $14.0\%$ and $23.3\%$ at the limit $p{=}31$.}
\label{fig:degree_penalty}
\end{figure*}

\begin{figure*}[!t]

\vspace{0.5em}
\includegraphics[width=1.1\textwidth]{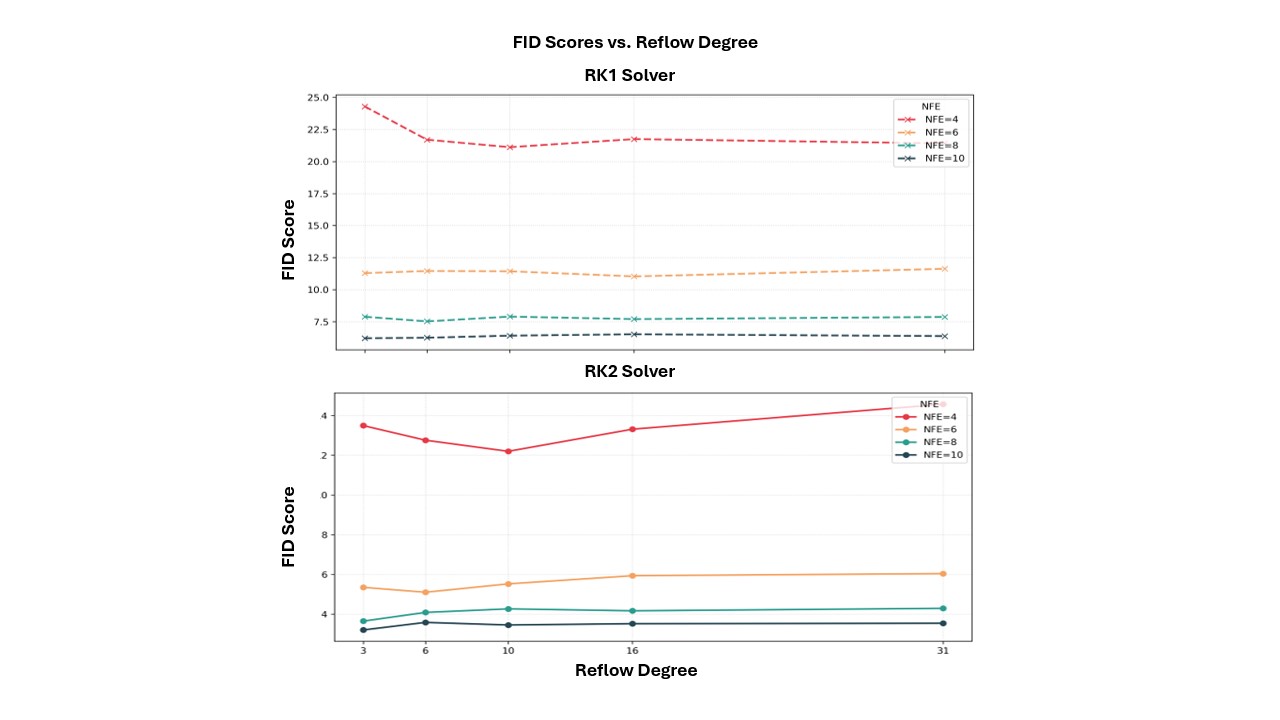}
\captionof{figure}{FID against I-spline degree on CIFAR-10 with ReFlow, $K{=}32$, for
RK1 (top) and RK2 (bottom) at NFE $4$--$10$. Degree $3$ is best at the
higher budgets (RK1 NFE${=}10$, RK2 NFE${=}8$ and $10$), where the
near-straight ReFlow trajectory leaves little curvature to fit, while
NFE${=}4$ prefers $p{=}10$ on both solvers and quality falls off toward
the global limit $p{=}31$. The best degree therefore depends on the model
and the budget, and only the I-spline can choose it: B\'ezierFlow is fixed
at $p=K-1=31$, the worst setting in six of eight columns \textbf{(Contribution~1)}.}
\label{fig:degree_fid_reflow}
\end{figure*}

\begin{figure*}[!t]

\subsection{Monotonicity}
\centering
\includegraphics[width=\textwidth]{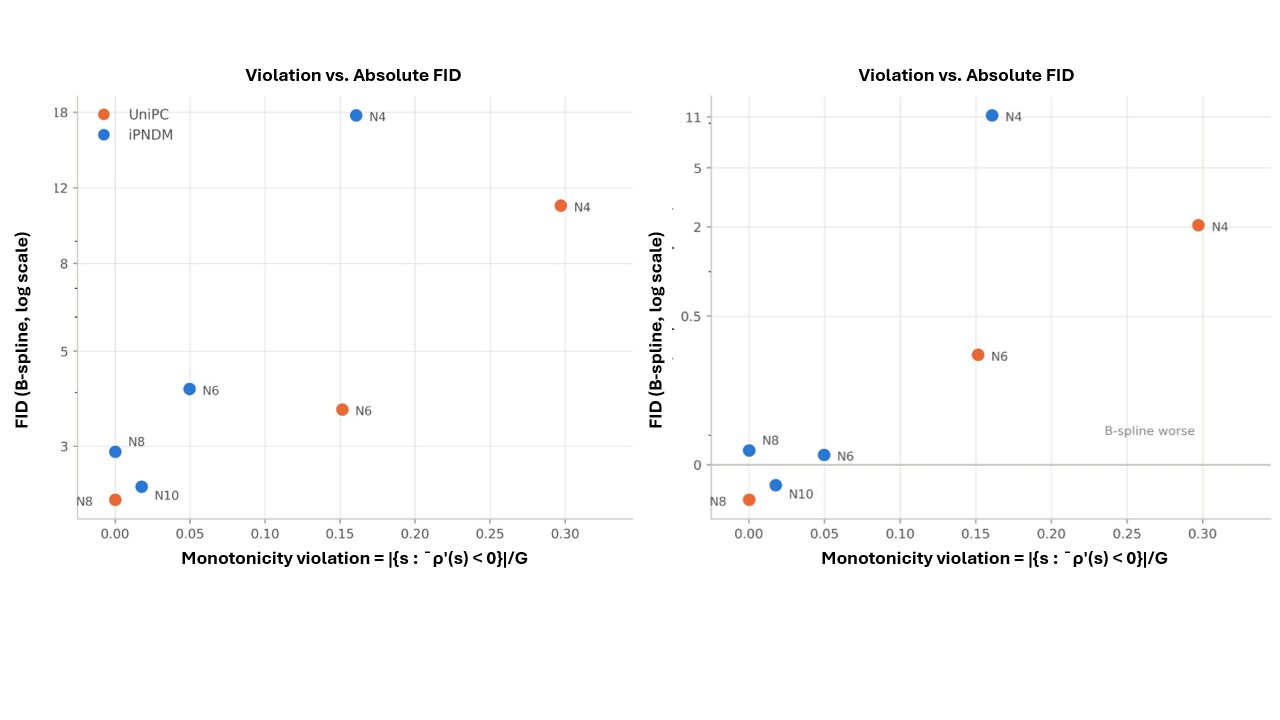}
\caption{The x axis is the fraction of the grid where the B-spline's SNR is
decreasing, $\nu=|\{s:\bar\rho'(s)<0\}|/G$; the I-spline has $\nu=0$ always
\textbf{(Theorem~\ref{thm:admissible-supp})}. (a) $\nu$ against absolute FID: the more
monotonicity the B-spline loses, the worse its FID. (b) $\nu$ against the FID
gap to the I-spline at the same solver and NFE: where the B-spline stays
monotone it matches the I-spline, and where it loses monotonicity it
degrades, by up to 11 FID. This is why the I-spline's built-in monotonicity
gives better FID.}
\label{fig:monotonicity}

\vspace{0.5em}
\subsection{Locality and Conditioning}
\small
\setlength{\tabcolsep}{5pt}
\captionof{table}{Wall-clock training time on CIFAR-10 with ReFlow, at $K{=}32$ and a
matched epoch budget. I-SplineFlow is at least as fast as B\'ezierFlow in
seven of eight settings, though the margins are small enough that we read
them as parity. Cost is set by the NFE budget, not the basis, so the FID
gains of the main paper come at no extra training cost. The
convergence advantage is measured in epochs instead, in the main paper.}
\label{tab:traintime}
\begin{tabular}{l c c c c}
\toprule
Solver & NFE & B\'ezierFlow & I-Spline $p{=}3$ & $\Delta$ \\
\midrule
RK1 (Euler)    & 4  & 23.58 & 21.17 & \textbf{$+2.41$} \\
               & 6  & 32.14 & 33.66 & $-1.53$ \\
               & 8  & 40.13 & 38.03 & \textbf{$+2.11$} \\
               & 10 & 50.37 & 46.91 & \textbf{$+3.45$} \\
\midrule
RK2 (Midpoint) & 4  & 22.62 & 21.85 & \textbf{$+0.77$} \\
               & 6  & 30.96 & 28.94 & \textbf{$+2.02$} \\
               & 8  & 39.79 & 38.80 & \textbf{$+0.99$} \\
               & 10 & 49.49 & 48.65 & \textbf{$+0.84$} \\
\bottomrule
\end{tabular}
\end{figure*}

\FloatBarrier

\end{document}